\documentclass{article}

\usepackage{geometry}
\usepackage{amsthm}
\usepackage[utf8]{inputenc} 
\usepackage[T1]{fontenc}    
\usepackage{color}
\definecolor{mydarkblue}{rgb}{0,0.03,0.15}
\usepackage[colorlinks=true, bookmarks=true,plainpages=false,pdfpagelabels=true, pageanchor=false,hypertexnames=false,linkcolor=mydarkblue,citecolor=mydarkblue]{hyperref}       
\usepackage{url}            
\usepackage{booktabs}       
\usepackage{amsfonts}       
\usepackage{nicefrac}       
\usepackage{microtype}      

\usepackage{comment}
\usepackage{macros}
\usepackage{tikz}
\usepackage{multirow}
\usepackage{rotating}
\usepackage{defs}
\usepackage{rl_var}
\usepackage[ruled,linesnumbered,boxed]{algorithm2e}
\usepackage[articlestyle]{mpemath}

\usetikzlibrary{shapes.arrows}
\usetikzlibrary{arrows,automata,backgrounds,positioning}

\newcommand{\polbase}{\pi_{\mathrm{B}}}	
\newcommand{\polrob}{\pi_{\mathrm{R}}}	

\newcommand{\polrobi}{\pi_{\mathrm{S}}}

\newcommand{\pol}{\pi}
\newcommand{\indist}{p_0}
\newcommand{\return}{\rho}
\newcommand{\disc}{\gamma}
\newcommand{\rmax}{R_{\max}}

\title{Safe Policy Improvement by Minimizing Robust Baseline Regret}
\author{Marek Petrik, Yinlam Chow, Mohammad Ghavamzadeh}

\begin{document}

\maketitle

\begin{abstract}
An important problem in sequential decision-making under uncertainty is to use limited data to  compute a {\em safe} policy, i.e.,~a policy that is guaranteed to perform at least as well as a given baseline strategy. In this paper, we develop and analyze a new {\em model-based} approach to compute a safe policy when we have access to an inaccurate dynamics model of the system with known accuracy guarantees. Our proposed robust method uses this (inaccurate) model to directly minimize the (negative) regret w.r.t.~the baseline policy. Contrary to the existing approaches, minimizing the regret allows one to improve the baseline policy in states with accurate dynamics and seamlessly fall back to the baseline policy, otherwise. We show that our formulation is NP-hard and propose an approximate algorithm. Our empirical results on several domains show that even this relatively simple approximate algorithm can significantly outperform standard approaches.  
\end{abstract}


\section{Introduction} 
\label{sec:intro}

Many problems in science and engineering can be formulated as a sequential decision-making problem under uncertainty. A common scenario in such problems that occurs in many different fields, such as online marketing, inventory control, health informatics, and computational finance, is to find a good or an optimal strategy/policy, given a batch of data generated by the current strategy of the company (hospital, investor). Although there are many techniques to find a good policy given a batch of data, only a few of them guarantee that the obtained policy will perform well, when it is deployed. Since deploying an untested policy can be risky for the business, the product (hospital, investment) manager does not usually allow it to happen, unless we provide her/him with some performance guarantees of the obtained strategy, in comparison to the baseline policy (e.g.,~the policy that is currently in use). 

In this paper, we focus on the \emph{model-based} approach to this fundamental problem in the context of \emph{infinite-horizon} discounted Markov decision processes (MDPs). In this approach, we use the batch of data and build a \emph{model} or a \emph{simulator} that approximates the true behavior of the dynamical system, together with an \emph{error function} that captures the accuracy of the model at each state of the system. Our goal is to compute a \emph{safe} policy, i.e.,~a policy that is guaranteed to perform at least as well as the baseline strategy, using the simulator and error function. Most of the work on this topic has been in the {\em model-free} setting, where {\em safe} policies are computed directly from the batch of data, without building an explicit model of the system~\cite{thomas2015high,Thomas2015a}. Another class of {\em model-free} algorithms are those that use a batch of data generated by the current policy and return a policy that is guaranteed to perform better. They optimize for the policy by repeating this process until convergence~\cite{kakade2002approximately,Pirotta2013}. 

A major limitation of the existing methods for computing safe policies is that they either adopt a newly learned policy with provable improvements or do not make any improvement at all by returning the baseline policy. These approaches may be quite limiting when model uncertainties are not uniform across the state space. In such cases, it is desirable to guarantee an improvement over the baseline policy by combining it with a learned policy on a state-by-state basis. In other words, we want to use the learned policy at the states in which either the improvement is significant or the model uncertainty (error function) is small, and to use the baseline policy everywhere else. However, computing a learned policy that can be effectively combined with a baseline policy is non-trivial due to the complex effects of policy changes in an MDP. Our key insight is that this goal can be achieved by minimizing the (negative) \emph{robust regret} w.r.t.~the baseline policy. This unifies the sources of uncertainties in the learned and baseline policies and allows a more systematic performance comparison. Note that our approach differs significantly from the standard one, which compares a pessimistic performance estimate of the learned policy with an optimistic estimate of the baseline strategy. That may result in rejecting a learned policy with a performance (slightly) better than the baseline, simply due to the discrepancy between the pessimistic and optimistic evaluations. 

The {\em model-based} approach of this paper builds on \emph{robust} Markov decision processes~\cite{Iyengar2005,Wiesemann2013,Ahmed2013}. The main difference is the availability of the baseline policy that creates unique challenges for sequential optimization. To the best of our knowledge, such challenges have not yet been fully investigated in the literature. A possible solution is to solve the robust formulation of the problem and then accept the resulted policy only if its conservative performance estimate is better than the baseline. While such an idea has been investigated in the {\em model-free} setting (e.g.,~\cite{Thomas2015a}), we show in this paper that such an approach is overly conservative. 

As the main contribution of the paper, we propose and analyze a new robust optimization formulation that captures the above intuition of minimizing robust regret w.r.t.~the baseline policy. After a preliminary discussion in Section~\ref{sec:prelimMDP}, we formally describe our  model and analyze its main properties in Section~\ref{sec:model}. We show that in solving this optimization problem, we may have to go beyond the standard space of deterministic policies and search in the space of randomized policies; we derive a bound on the performance loss of its solutions; and we prove that solving this problem is NP-hard. We also propose a simple and practical approximate algorithm. Then, in Section~\ref{sec:algorithms}, we show that the standard model-based approach is really a tractable approximation of robust baseline regret minimization. Finally, our experimental results in Section~\ref{sec:experiments} indicate that even the simple approximate algorithm significantly outperforms the standard model-based approach when the model is uncertain. 

\section{Preliminaries}  \label{sec:prelimMDP}

We consider problems in which the agent's interaction with the environment is modeled as an \emph{infinite-horizon} $\gamma$-discounted MDP. A $\gamma$-discounted MDP is a tuple $\M=\langle\X,\A,r,P,p_0,\gamma\rangle$, where $\X$ and $\A$ are the state and action spaces, $r(x,a)\in[-R_{\max},R_{\max}]$ is the bounded reward function, $P(\cdot|x,a)$ is the transition probability function, $p_0(\cdot)$ is the initial state distribution, and $\gamma\in (0,1]$ is a discount factor. We use $\Polrand = \{\pol : \X \rightarrow \Delta^{\A} \}$ and $\Poldet= \{\pol : \X \rightarrow \A \}$ to denote the sets of {\em randomized} and {\em deterministic} stationary Markovian policies, respectively, where $\Delta^{\A}$ is the set of probability distributions over the action space $\A$. 

%
%

Throughout the paper, we assume that the true reward $r$ of the MDP is known, but the true transition probability is not given. The generalization to include reward estimation is straightforward and is omitted for the sake of brevity. We use historical data to build a MDP \emph{model} with the transition probability denoted by $\widehat{P}$. Due to limited number of samples and other modeling issues, it is unlikely that $\widehat{P}$ matches the true transition probability of the system $P\opt$. We also require that the estimated model $\widehat{P}$ deviates from the true transition probability $P\opt$ as stated in the following assumption:

\begin{assumption}  
\label{asm:error}
For each $(x,a)\in\mathcal{X}\times \mathcal{A}$, the error function $e(x,a)$ bounds the $\ell_1$ difference between the estimated transition probability and true transition probability, i.e., 
\begin{equation}
\label{eq:ass1} 
\|P\opt(\cdot|x,a)-\widehat{P}(\cdot|x,a)\|_1 \le e(x,a).
\end{equation}
\end{assumption}

The error function $e$ can be derived either directly from samples using high probability concentration bounds, as we briefly outline in~\cref{subsec:sampling_bounds}, or based on specific domain properties.

To model the uncertainty in the transition probability, we adopt the notion of robust MDP (RMDP)~\cite{Iyengar2005,Nilim2005,Wiesemann2013}, i.e.,~an extension of MDP in which nature adversarially chooses the transitions from a given {\em uncertainty set}
\begin{align}
\label{eq:uncertainty_set}
\Nat(\widehat{P},e) = \Big\{\nat: \X\times\A\rightarrow \Delta^{\X}: \|\nat(\cdot|x,a) - \widehat{P}(\cdot|x,a)\|_1\leq e(x,a),\;\forall x,a\in\X\times\A\Big\}~. \nonumber 
\end{align}
From Assumption~\ref{asm:error}, we notice that the true transition probability is in the set of uncertain transition probabilities, i.e., $P^\star\in\Nat(\widehat{P},e)$. The above $\ell_1$ constraint is common in the RMDP literature~(e.g.,~\cite{Iyengar2005, Petrik2014, Wiesemann2013}). The uncertainty set $\Nat$ in RMDP is $(x,a)$-rectangular and randomized~\cite{LeTallec2007,Wiesemann2013}. One of the motivations for considering $(x,a)$-rectangular sets in RMDP is that they lead to tractable solutions in the conventional reward maximization setting. However, in the robust regret minimization problem that we propose in this paper, even if we assume that the uncertainty set is $(x,a)$-rectangular, it does not guarantee tractability of the solution. While it is of great interest to investigate the structure of uncertainty sets that lead to tractable algorithms in robust regret minimization, it is beyond the main scope of this paper and we leave it as future work.

For each policy $\pi\in\Pi_R$ and nature's choice $\xi\in\Xi$, the discounted {\em return} is defined as 
%
\begin{align*}
\rho(\pi,\nat) &= \lim_{T\rightarrow \infty} \Ex{\nat}{\sum_{t=0}^{T-1}\disc^t r \bigl(X_t,A_t\bigr) \mid X_0\sim p_0,\;A_t\sim \pol(X_t)} = p_0^\top \val_\pol^{\nat},
\end{align*}
%
where $X_t$ and $A_t$ are the state and action random variables at time $t$, and $\val_\pol^{\nat}$ is the corresponding {\em value function}. An {\em optimal policy} for a given $\xi$ is defined as $\pol\opt_{\nat}\in\argmax_{\pi\in\Polrand}\rho(\pi,\nat)$.
%
Similarly, under the true transition probability $P^\star$, the {\em true return} of a policy $\pol$ and a {\em truly optimal policy} are defined as $\return(\pol,P^\star)$ and $\pol\opt \in \arg\max_{\pol\in\Polrand} \return(\pol,P^\star)$, respectively. Although we define the optimal policy using $\arg\max_{\pi\in\Polrand}$, it is known that every reward maximization problem in MDPs has at least one optimal policy in $\Poldet$.

Finally, given a deterministic {\em baseline} policy $\pi_B$, we call a policy $\pi$ {\em safe}, if its "true" performance is guaranteed to be no worse than that of the baseline policy, i.e.,~$\rho(\pi,P^\star)\ge\rho(\pi_B,P^\star)$.


\section{Robust Policy Improvement Model} 
\label{sec:model}

In this section, we introduce and analyze an optimization procedure that robustly improves over a given baseline policy $\pi_B$. As described above, the main idea is to find a policy that is guaranteed to be an improvement for any realization of the uncertain model parameters. The following definition formalizes this intuition. 

\begin{definition}[The Robust Policy Improvement Problem] 
\label{def:robust_interleave}
Given a model uncertainty set $\Nat(\widehat{P},e)$ and a baseline policy $\polbase$, find a maximal $\zeta \ge 0$ such that there exists a policy $\pol\in\Polrand$ for which $\return(\pol,\nat) \ge \return(\polbase,\nat)+\zeta$, for every $\nat\in\Nat(\widehat{P},e)$.\footnote{From now on, for brevity, we omit the parameters $\widehat{P}$ and $e$, and use $\Nat$ to denote the model uncertainty set.}
\end{definition}

The problem posed in Definition~\ref{def:robust_interleave} readily translates to the following optimization problem:
\begin{equation} 
\label{eq:objective_robust_interleave}
\polrobi \in \arg \max_{\pol \in\Polrand} \min_{\nat\in\Nat} \; \Bigl( \return(\pol, \nat) - \return(\polbase, \nat)\Bigr).
\end{equation}
%

Note that since the baseline policy $\polbase$ achieves value $0$ in~\eqref{eq:objective_robust_interleave}, $\zeta$ in Definition~\ref{def:robust_interleave} is always non-negative. Therefore, any solution $\polrobi$ of~\eqref{eq:objective_robust_interleave} is {\em safe}, because under the true transition probability $P\opt\in\Nat(\widehat{P},e)$, we have the guarantee that $\return(\pol, P\opt) - \return(\polbase, P\opt) \ge \min_{\nat\in\Nat} \; \Bigl( \return(\pol, \nat) - \return(\polbase, \nat) \Bigr) \ge 0$. It is important to highlight how Definition~\ref{def:robust_interleave} differs from the standard approach (e.g.,~\cite{Thomas2015a}) on determining whether a policy $\pi$ is an improvement over the baseline policy $\pi_B$. The standard approach considers a statistical error bound that translates to the test: $\min_{\nat\in\Nat} \return(\pol,\nat) \ge \max_{\nat\in\Nat} \return(\polbase,\nat)$. Note that the uncertainty parameters $\nat$ on both sides of the above inequality are not necessarily the same. Therefore, any optimization procedure derived based on this test is more conservative than the problem in~\eqref{eq:objective_robust_interleave}. Indeed when the error function in $\Xi$ is large, even the baseline policy ($\pol=\polbase$) may not pass this test. In Section~\ref{sub:regret_example}, we show the conditions under which this approach fails.

In the remainder of this section, we highlight some major properties of the optimization problem~\eqref{eq:objective_robust_interleave}. Specifically, we show that its solution policy may be purely randomized, we compute a bound on the performance loss of its solution policy w.r.t.~$\pi^\star$, and we finally prove that it is a NP-hard problem.

\subsection{Policy Class}

The following theorem shows that we should search for the solutions of the optimization problem~\eqref{eq:objective_robust_interleave} in the space of randomized policies $\Pi_R$.  

\begin{theorem} 
\label{thm:randomized_optimal}
The solution to the optimization problem~\eqref{eq:objective_robust_interleave} may not be attained by a deterministic policy. Moreover, the loss due to considering deterministic policies cannot be bounded, i.e.,~there exists no constant $c \in \Real$ such that
\begin{equation*}
\max_{\pol \in\Polrand}\min_{\nat\in\Nat} \; \Bigl( \return(\pol, \nat) - \return(\polbase, \nat)\Bigr) \le c\cdot\max_{\pol \in\Poldet} \min_{\nat\in\Nat} \; \Bigl( \return(\pol, \nat) - \return(\polbase, \nat)\Bigr).
\end{equation*}
\end{theorem}
\begin{proof}
The proof follows directly from Example~\ref{exm:optimal_policy_stochastic}. The optimal policy in this example is randomized and achieves a guaranteed improvement $\zeta=1/2$. There is no deterministic policy that guarantees a positive improvement over the baseline policy, which proves the second part of the theorem.
\end{proof}

\begin{figure}
\begin{tabular}{cc}
	\begin{tikzpicture}[->,>=stealth',shorten >=1pt,auto,node distance=1.4cm,semithick]
	\tikzstyle{action}= [circle,fill=black,text=white]
	
	\node[state]        (s1)       					{$x_1$};
	\node[action]      (a1)         [above right of=s1]    {$a_1$};
	\node[action]      (a2)         [below right of=s1]    {$a_2$};
	
	\path (s1) edge[double] (a1) edge (a2);
	
	\node (o2) [right of=a2] {$2$};
	\path (a2) edge (o2);
	
	\node[state] (s11) [above right of=a1,xshift=0.7cm] {$x_{11}$};
	\node (s12) [below right of=a1,xshift=0.7cm] {$1$};
	
	\path (a1) edge node[above] {$\xi_1$} (s11);
	\path (a1) edge node[below] {$\xi_2$} (s12);
	
	\node (o111) [right of=s11,yshift=0.5cm]	{2};
	\node (o112) [right of=s11,yshift=-0.5cm]	{3};
	\path (s11) edge[double] node[above] {$a_{11}$} (o111) edge node[below] {$a_{12}$} (o112);
	
	\end{tikzpicture} \qquad\qquad\qquad
&	
	\begin{tikzpicture}[->,>=stealth',shorten >=1pt,auto,node distance=1.4cm,semithick]
	\tikzstyle{action}= [circle,fill=black,text=white]
	
	\node[state,initial] 		 (s1)                      {$x_0$};
	\node[action]      (a1)         [above right of=s1]    {$a_1$};
	\node[action]      (a2)         [below right of=s1]    {$a_2$};
	
	\node[state] (s2) [below right of=a1] {$x_1$};
	
	\node[action] (a11) [right of=s2,xshift=0mm] {$a_1$};
	
	\node (o1) [above right of=a11] {$+10/\gamma$};
	\node (o2) [below right of=a11] {$-10/\gamma$};
	
	\path (s1) edge node [above left]{$\pi_{\text{B}}$}  (a1) ;
	\path (a1) edge node[above right] {0} (s2);
	\path (a11) edge node [above left]{$\xi^\star$} (o1);
	\path (a11) edge node [below left]{$\xi_1$} (o2);
	
	\path (s1) edge node [below left]{$\pi^\star$}  (a2) ;
	\path (a2) edge node[below right] {1} (s2);
	
	\path (s2) edge node[above] {$\pi^\star$} node[below] {$\pi_{\text{B}}$} (a11);
	\end{tikzpicture}
\end{tabular}	
	\caption{\textbf{\textit{(left)}} A robust/uncertain MDP used in Example~\ref{exm:optimal_policy_stochastic} that illustrates the sub-optimality of deterministic policies in solving the optimization problem~\eqref{eq:objective_robust_interleave}. \textbf{\textit{(right)}} A Markov decision process with significant uncertainty in baseline policy.} \label{fig:optimal_stochastic}
\end{figure}
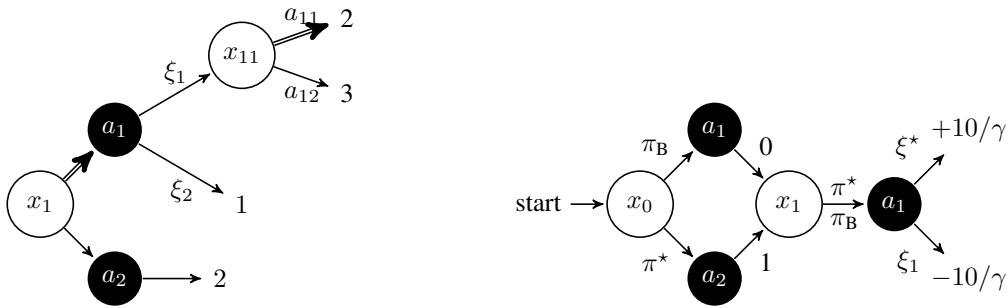

\begin{example} 
\label{exm:optimal_policy_stochastic}
Consider the robust/uncertain MDP on the left panel of~Figure~\ref{fig:optimal_stochastic} with states $\{x_1,x_{11}\}\subset\X$, actions $\A=\{a_1, a_2, a_{11}, a_{12}\}$, and discount factor $\gamma=1$. Actions $a_1$ and $a_2$ are shown as solid black nodes. A number with no state represents a terminal state with the corresponding reward. The robust outcomes $\{\xi_1,\xi_2\}$ correspond to the uncertainty set of transition probabilities $\Nat$. The baseline policy $\polbase$ is deterministic and is denoted by double edges. It can be readily seen from the monotonicity of the Bellman operator that any improved policy $\pol$ will satisfy $\pol(a_{12}|x_{11}) = 1$. Therefore, we will only focus on the policy at state $x_1$. The robust improvement as a function of $\pol(\cdot|x_1)$ and the uncertainties $\{\xi_1,\xi_2\}$ is given as follows:
\begin{align*}
\min_{\xi\in\Xi}\big(\return(\pol,\nat)-\return(\polbase,\nat)\big)
&=\min_{\xi\in\Xi}\left(\left[ \begin{array}{c|cc}
\pol~\backslash~\nat & \xi_1 & \xi_2 \\
\hline
a_1 & 3 & 1 \\
a_2 & 2 & 2 
\end{array} \right] - \left[\begin{array}{c|cc}
\pol~\backslash~\nat & \xi_1 & \xi_2 \\
\hline
a_1 & 2 & 1 
\end{array} \right]\right) \\ 
&=\min_{\xi\in\Xi}\left[\begin{array}{c|cc}
\pol~\backslash~\nat & \xi_1 & \xi_2 \\
\hline
a_1 & 1 & 0 \\
a_2 & 0 & 1 
\end{array} \right] = 0.
\end{align*}
This shows that no deterministic policy can achieve a positive improvement in this problem. However, a randomized policy $\pol(a_1|x_1)=\pol(a_2|x_1)=1/2$ returns the maximum improvement $\zeta=1/2$.
\end{example}

Randomized policies can do better than their deterministic counterparts, because they allow for hedging among various realizations of the MDP parameters. Example~\ref{exm:optimal_policy_stochastic} shows a problem such that there exists a realization of the parameters with improvement over the baseline when any deterministic policy is executed. However in this example, there is no single realization of parameters that provides an improvement for all the deterministic policies \emph{simultaneously}. Therefore, randomizing the policy guarantees an improvement independent of the parameters' choice. 

\subsection{Performance Bound}

Generally, one cannot compute the truly optimal policy $\pi^\star$ using an imprecise model. Nevertheless, it is still crucial to understand how errors in the model translates to a performance loss w.r.t.~an optimal policy. The following theorem provides a bound on the performance loss of any solution $\pi_S$ to the optimization problem~\eqref{eq:objective_robust_interleave}.

\begin{theorem}
\label{thm:perf_robust_interleave}
A solution $\polrobi$ to the optimization problem~\eqref{eq:objective_robust_interleave} is safe and its performance loss is bounded by the following inequality:
\begin{align*}
\Phi(\polrobi)&\stackrel{\Delta}{=}\return(\pol\opt,P^\star) - \return(\polrobi,P^\star) \le\min\! \left\{\! \frac{2 \gamma \rmax}{(1-\disc)^2}  \Big( \| e_{\pol\opt} \|_{1,u^\star_{\pi^\star}}  \!\!+\! \| e_{\polbase} \|_{1,u^\star_{\pi_B}} \!\Big),  \Phi(\polbase) \!\right\}, 
\end{align*}
where $u\opt_{\pi^\star}$ and $u^\star_{\pi_B}$ are the state occupancy distributions of the optimal and baseline policies in the true MDP $P^\star$. Furthermore, the above bound is tight. 
\end{theorem}

The proof of Theorem~\ref{thm:perf_robust_interleave} is available in Appendix~\ref{sec:proof-bound}. 

\subsection{Computational Complexity}

In this section, we analyze the computational complexity of solving the optimization problem~\eqref{eq:objective_robust_interleave} and prove that the problem is NP-hard. In particular, we proceed by showing that the following sub-problem of~\eqref{eq:objective_robust_interleave}, for a fixed $\pol\in\Polrand$, is NP-hard:
\begin{equation}  
\label{eq:obj_interleave_inner}
\arg\min_{\nat\in\Nat} ~\Bigl( \return(\pol, \nat) - \return(\polbase, \nat)\Bigr).
\end{equation}
The optimization problem~\eqref{eq:obj_interleave_inner} can be interpreted as computing a policy that simultaneously minimizes the returns of two MDPs, whose transitions induced by policies $\pol$ and $\polbase$. The proof of Theorem~\ref{thm:complexity_joint} is given in Appendix~\ref{sec:proof-complexity}.

\begin{theorem} 
\label{thm:complexity_joint}
Both optimization problems~\eqref{eq:objective_robust_interleave} and~\eqref{eq:obj_interleave_inner} are NP-hard.
\end{theorem}

Although the optimization problem~\eqref{eq:objective_robust_interleave} is NP-hard in general, but it can be tractable under certain conditions. The following proposition shows that this is the case, for example, when the Markov chain induced by the baseline policy is known precisely. 

\begin{proposition} 
\label{prop:tractable_known}
Assume that for each $x\in\X$, the error function induced by the baseline policy is zero, i.e., $e\big(x,\pi_B(x)\big)=0$.\footnote{Note that this is equivalent to precisely knowing the Markov chain induced by the baseline policy $P^\star_{\pi_B}$.} Then, the optimization problem~\eqref{eq:objective_robust_interleave} is equivalent to the following problem and can be solved in polynomial time:
\begin{equation} \label{eq:exact1}
\arg \max_{\pi \in\Polrand} \min_{\nat\in\Nat} \; \return(\pi, \nat).
\end{equation}
\end{proposition}
\begin{proof}
The hypothesis in the proposition implies that for any $\xi\in\Xi(\widehat{P},e)$, we have $\xi\big(\cdot|x,\pi_B(x)\big) = \widehat{P}\big(\cdot|x,\pi_B(x)\big)$, $\forall x\in\X$. This further indicates that $\rho(\pi_B,\xi)$ is a constant (independent of $\xi$), for all $\xi\in\Xi(\widehat{P},e)$. Thus, when the Markov chain induced by the baseline policy is known, the optimization problem~\eqref{eq:objective_robust_interleave} is reduced to the optimization problem~\eqref{eq:exact1}, which is a robust MDP (RMDP) problem with $\ell_1$-constraint uncertainty set. It is known that this class of RMDP problems can be solved in (strongly) polynomial time~\cite{Hansen2013} and has also been solved efficiently in practice~\cite{Petrik2014}.
\end{proof}

\subsection{Approximate Algorithm} \label{sub:heuristic}

Solving for the optimal solution of~\eqref{eq:objective_robust_interleave} may not be possible in practice since the problem is NP hard. In this section, we propose a simple and practical approximate algorithm. The empirical results of~\cref{sec:experiments} indicate that this algorithm holds promise and they also suggest that the approach may be a good starting point for building better approximate algorithms in the future.

\begin{algorithm}
	\SetKwInOut{Input}{input}\SetKwInOut{Output}{output}
	\Input{Empirical transition probabilities: $\widehat{P}$, baseline policy $\pi_B$, and the error function $e$} 
	\Output{Policy $\tilde{\pi}_S$}
	\ForEach{$x\in\mathcal{X},a\in\mathcal{A}$}{
		$\tilde{e}(x,a) \leftarrow \begin{cases} 
		e(x,a)  &\text{when } \polbase(x) \neq a \\ 
		0  &\text{otherwise} \end{cases}$ \;}
	$\tilde{\pi}_S \leftarrow \argmax_{\pi\in\Pi_R} \min_{\nat\in\Nat(\widehat{P},\tilde{e})} \left( \rho\big(\pol,\nat\big)  - \rho\big(\polbase,\nat\big)\right)$ \;
	\Return $\tilde{\pi}_S$
	\caption{Approximate Robust Baseline Regret Minimization Algorithm}  \label{alg:approximate_robust_joint}
\end{algorithm}

\cref{alg:approximate_robust_joint} contains the pseudocode of the proposed approximate method. The main idea is to use a modified uncertainty model by assuming no error in transition probabilities of the baseline policy. Then it is possible to minimize the robust baseline regret in polynomial time as \cref{prop:tractable_known} shows. Assuming no error in baseline transition probabilities is reasonable because of two main reasons. First, data is in practice often generated by executing  the baseline policy and therefore we may have enough data for a good approximation its transition probabilities: ~$\forall x\in\X,\widehat{P}\big(\cdot|x,\pi_B(x)\big) \approx P^\star\big(\cdot|x,\pi_B(x)\big)$. Second, transition probabilities often affect baseline and improved policies similarly and therefore have little effect on the difference between their returns (i.e., the regret). See \cref{sub:regret_example} for an example of such behavior.



\newcommand{\vbase}{v_{\text{B}}}
\newcommand{\qbase}{v_{\text{B}}}
\newcommand{\uset}[1]{\mathcal{U}(#1)}
\section{Standard Policy Improvement Methods} 
\label{sec:algorithms}

In Section~\ref{sec:model}, we showed that finding an exact solution to the optimization problem~\eqref{eq:objective_robust_interleave} is computationally expensive and proposed an approximate algorithm. In this section, we describe and analyze two standard methods for computing safe policies and show how they can be interpreted as an approximation of our proposed baseline regret minimization. Due to space limitations, we describe another method, called reward-adjusted MDP, in Appendix~\ref{sec:reward-adjusted}, but report its performance in Section~\ref{sec:experiments}.



\subsection{Solving the Simulator}
\label{subsub:simulator}

The most straightforward solution to~\eqref{eq:objective_robust_interleave} is to simply assume that our {\em simulator} is accurate and to solve the reward maximization problem of a MDP with the transition probability $\widehat{P}$, i.e.,~$\pi_{\text{sim}}\in\argmax_{\pi\in\Polrand}\rho(\pi,\widehat{P})$.
%
Theorem~\ref{thm:perf-simulator} quantifies the performance loss of the resulted policy $\pi_{\text{sim}}$.

\begin{theorem}  
\label{thm:perf-simulator}
Let $\pi_{\text{sim}}$ be an optimal policy of the reward maximization problem of a MDP with transition probability $\widehat{P}$. Then under Assumption~\ref{asm:error}, the performance loss of $\pi_{\text{sim}}$ is bounded by
\begin{equation*}
\Phi(\pi_{\text{sim}})\stackrel{\Delta}{=}\return(\pi^\star,P^\star) - \return(\pi_{\text{sim}},P^\star) \le \frac{2 \gamma \rmax}{(1-\disc)^2}  \| e \|_\infty.  
\end{equation*}
\end{theorem}

The proof is available in Appendix~\ref{sec:proof-perf-simulator}. Note that there is no guarantee that $\pi_{\text{sim}}$ is \emph{safe}, and thus, deploying it may lead to undesirable outcomes due to model uncertainties. Moreover, the performance guarantee of $\pi_{\text{sim}}$, reported in Theorem~\ref{thm:perf-simulator}, is weaker than that in Theorem~\ref{thm:perf_robust_interleave} for the solution to our proposed optimization problem~\eqref{eq:objective_robust_interleave}.

\cref{thm:perf_robust} indicates that the policy $\pi_R$ returned by Algorithm~\ref{alg:robust} is {\em safe} and has a tighter bound on its performance loss than $\pi_{\text{sim}}$. This is because Theorem~\ref{thm:perf_robust} depends on a weighted $\ell_1$-norm of the errors in the optimal policy, instead of the $\ell_\infty$-norm over all policies in Theorem~\ref{thm:perf-simulator}. 

\subsection{Solving Robust MDP}
\label{subsubsec:robust}

Another standard solution to the problem in~\eqref{eq:objective_robust_interleave} is based on solving the RMDP problem~\eqref{eq:exact1}. We prove that the policy returned by this algorithm is {\em safe} and has better (sharper) worst-case guarantees than the simulator-based policy $\pi_{\text{sim}}$. Details of this algorithm are summarized in Algorithm~\ref{alg:robust}. The algorithm first constructs and solves an RMDP. It then returns the solution policy if its worst-case performance over the uncertainty set is better than the robust performance $\max_{\nat\in\Nat}\rho(\pi_B,\nat)$, and it returns the baseline policy $\pi_B$, otherwise. 

\begin{algorithm}
\SetKwInOut{Input}{input}\SetKwInOut{Output}{output}
\Input{Simulated MDP $\widehat{P}$, baseline policy $\pi_B$, and the error function $e$} 
\Output{Policy $\polrob$}
$\pi_0 \leftarrow \argmax_{\pi\in\Pi_R} \min_{\nat\in\Nat(\widehat{P},e)}\rho\big(\pi,\nat\big)$ \;
\leIf{$\min_{\nat\in\Nat(\widehat{P},e)} \rho\big(\pi_0,\nat\big) > \max_{\nat\in\Nat}\rho(\pi_B,\nat)$}{
\Return $\pi_0$}{\Return $\polbase$ \label{ln:algrob_condition}}
\caption{RMDP-based Algorithm} 
\label{alg:robust}
\end{algorithm}

Algorithm~\ref{alg:robust} makes use of the following approximation to the solution of~\eqref{eq:objective_robust_interleave}: 
\begin{equation*}
\max_{\pi \in\Pi_R} \min_{\nat\in\Nat} \Bigl( \return(\pi, \nat) - \return(\pi_B, \nat)\Bigr)\geq \max_{\pi \in\Pi_R} \min_{\nat\in\Nat} \return(\pi, \nat) -  \max_{\nat\in\Nat}\return(\pi_B, \nat), 
\end{equation*}
and guarantees safety by designing $\pi$ such that the RHS of this inequality is always non-negative.

The performance bound of $\polrob$ is identical to that in Theorem~\ref{thm:perf_robust_interleave}, and is stated and proved in \cref{thm:perf_robust} in Appendix~\ref{sec:proof-perf-robust}. However even though the worst-case bounds are the same, we show in \cref{sub:regret_example} that the performance loss of $\polrob$ may be worse than $\polrobi$ by an arbitrarily large margin.

It is important to discuss the difference between \cref{alg:approximate_robust_joint,alg:robust}. Although both solve an RMDP, they use different uncertainty sets $\Xi$. The uncertainty set used in \cref{alg:robust} is the true error function in building the simulator, while the uncertainty set used in \cref{alg:approximate_robust_joint} assumes that the error function is zero for all the actions suggested by the baseline policy. As a result, both algorithms approximately solve \eqref{eq:objective_robust_interleave} but approximate the problem in different ways.

\section{Experimental Evaluation} \label{sec:experiments}

In this section, we experimentally evaluate the benefits of minimizing the robust baseline regret. First, we demonstrate that solving the problem in \eqref{eq:objective_robust_interleave} may outperform the regular robust formulation by an arbitrarily large margin. Then, in the remainder of the section, we compare the solution quality of \cref{alg:approximate_robust_joint} with simpler methods in more complex and realistic experimental domains. The purpose of our experiments is to show how solution quality depends on the degree of model uncertainties.

\subsection{An Illustrative Example} \label{sub:regret_example}

%
%
%
%
%
%
%

Consider the example depicted on the right panel of Figure~\ref{fig:optimal_stochastic}. White nodes represent states and black nodes represent state-action pairs. Labels on the edges originated from states indicate the policy according to which the action is taken; labels on the edges originated from actions denote the rewards and, if necessary, the name of the uncertainty realization. The baseline policy is $\polbase$, the optimal policy is $\pol\opt$, and the discount factor is $\gamma \in (0,1)$.

This example represents a setting in which the level of uncertainty varies significantly across the individual states: the transitions model is precise in state $x_0$ and uncertain in state $x_1$. The baseline policy $\polbase$ takes a suboptimal action in state $x_0$ and the optimal action in the uncertain state $x_1$. To prevent being overly conservative in computing a safe policy, one needs to consider that the realization of uncertainty in $x_1$ influences both the baseline and improved policies. 

Using the plain robust optimization formulation in Algorithm \ref{alg:robust}, even the optimal policy $\pol\opt$ is not considered safe in this example. In particular, the robust return of $\pol\opt$ is $\min_\xi \return(\pol\opt,\xi) = -9$, while the optimistic return of $\polbase$ is $\max_\xi \return(\polbase,\xi) = +10$. On the other hand, solving \eqref{eq:objective_robust_interleave} will return the optimal policy since: 
$  \min_\xi \return(\pol\opt,\xi) - \return(\polbase,\xi) = 11 - 10 = -9 - (-10) = 1$. Even the heuristic method of Section~\ref{sub:heuristic} will return the optimal policy. Note that since the reward-adjusted formulation (see its description in~\cref{sec:reward-adjusted}) is even more conservative than the robust formulation, it will also fail to improve on the baseline policy. 

\subsection{Example Grid Problem}

In this section, we use a simple grid problem to compare the solution quality of \cref{alg:approximate_robust_joint} with simpler methods. The grid problem is motivated by modeling customer interactions with an online system. States in the problem represent a two dimensional grid. Columns capture states of interaction with the website, and rows capture customer states such as overall satisfaction. Actions can move customers along either dimension with some probability of failure. A more detailed description of this domain is provided in Section \ref{app:domains}.

Our goal is to evaluate how the solution quality of the various methods depends on the magnitude of model error $e$. The model is constructed from samples and thus the magnitude of the error depends on the number of samples used to build the model. We use a uniform random policy to gather samples. Model error function $e$ is then constructed from this simulated data using bounds in Section \ref{sec:tech-lemma}. The baseline policy is constructed to be optimal when ignoring the row part of state; see Section \ref{app:domains} for more details. 

All methods are compared in terms of the improvement percentage in total return over the baseline policy. \cref{fig:result} depicts the results as a function of the number of transition samples used in constructing the uncertain model and represent the mean of $40$ runs. Methods used in the comparison are as follows: 1) EXP represents solving the nominal model as described in \cref{subsub:simulator}, 2) RWA represent the reward-adjusted formulation in \cref{alg:reward-adjusted}, 3) ROB represents the robust method in \cref{alg:robust}, and 4) RBC represents the approximate algorithm in \cref{alg:approximate_robust_joint}.  

\begin{figure}
	\centering
	\vspace{-0.1in}
	\includegraphics[width=0.45\linewidth]{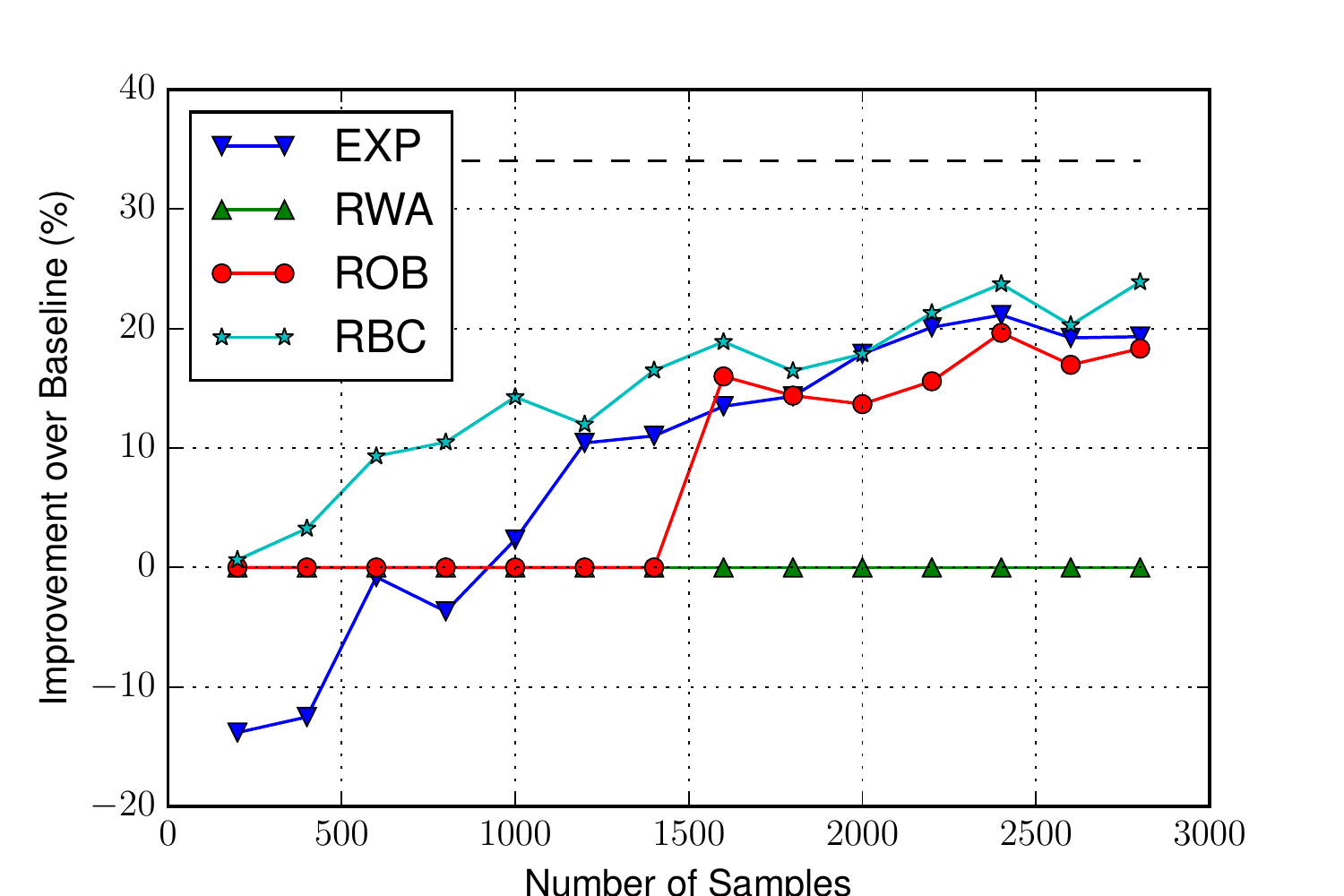}
	\vspace{-0.1in}
	\caption{Improvement in return over the baseline policy for the proposed methods. The dashed line shows the return of the optimal policy. } 
	\label{fig:result}
	\vspace{-0.15in}
\end{figure}

\cref{fig:result} shows that \cref{alg:approximate_robust_joint} not only reliably computes policies that are safe, but also significantly improves on the quality of the baseline policy when the model error is large. When the number of samples is small, \cref{alg:approximate_robust_joint} is significantly better than other methods by relying on the baseline policy in states with a large model error and only taking improving actions when the model error is small. Note that EXP  can be significantly worse that the baseline policy, especially when the number of samples is small.

\vspace{-0.1in}
\subsection{Energy Arbitrage}
\vspace{-0.1in}
In this section, we compare model-based policy improvement methods using a more complex domain. The problem is to determine an energy arbitrage policy in given limited energy storage (a battery) and stochastic prices. At each time period, the decision maker observes the available battery charge and a Markov state of energy price and decides on the amount of energy to purchase or to sell. 

The set of states in the energy arbitrage problem consists of three components: current state of charge, current capacity, and a Markov state representing price; the actions represent the amount of energy purchased or sold; the rewards indicate profit/loss in the transactions. We discretize the state of charge and action sets to 10 separate levels. The problem is based on the domain from \cite{Petrik2015} whose description is detailed in Appendix \ref{app:energy_arbitrage}. 

Energy arbitrage is a good fit for model-based approaches because it combines known and unknown dynamics. Physics of battery charging and discharging can be modeled with high confidence, while the evolution of energy prices is uncertain. As a result, using an explicit battery model the only uncertainty is in transition probabilities between the 10 states of the price process instead of the entire 1000 state-action pairs. This significantly reduces the number of samples needed to compute a good solution.

A realistic baseline policy is constructed by solving a high-precision version of the discretized problem in which the price process is aggregated to 3 levels from 10. This baseline policy represents a realistic but simplified solution. Because low energy prices are more commonly sampled than high energy prices, the degree of uncertainty varies significantly over the state space.

As in the previous application, we estimate the uncertainty model in a data-driven manner. Notice that the inherent uncertainty is only on price transitions, and is independent to the policy used (which controls the storage dynamics). Here the uncertainty set of transition probabilities is estimated by the method in Section \ref{subsec:sampling_bounds} but the uncertainty set is only a non-singleton with respect to price states. Figure \ref{fig:energy_results} shows the percentage improvement on the baseline policy averaged over 5 runs, whose labels of policies follow the definitions of Figure \ref{fig:result}. We clearly observe that the heuristic RBC method---described in Section \ref{sub:heuristic}---effectively interleaves the baseline policy (in states with low level of uncertainty) and an improved policy (in states with low level of uncertainty) and results in the best performance in most cases. 

\begin{figure}
	\centering
	\vspace{-0.15in}
	\includegraphics[width=0.45\linewidth]{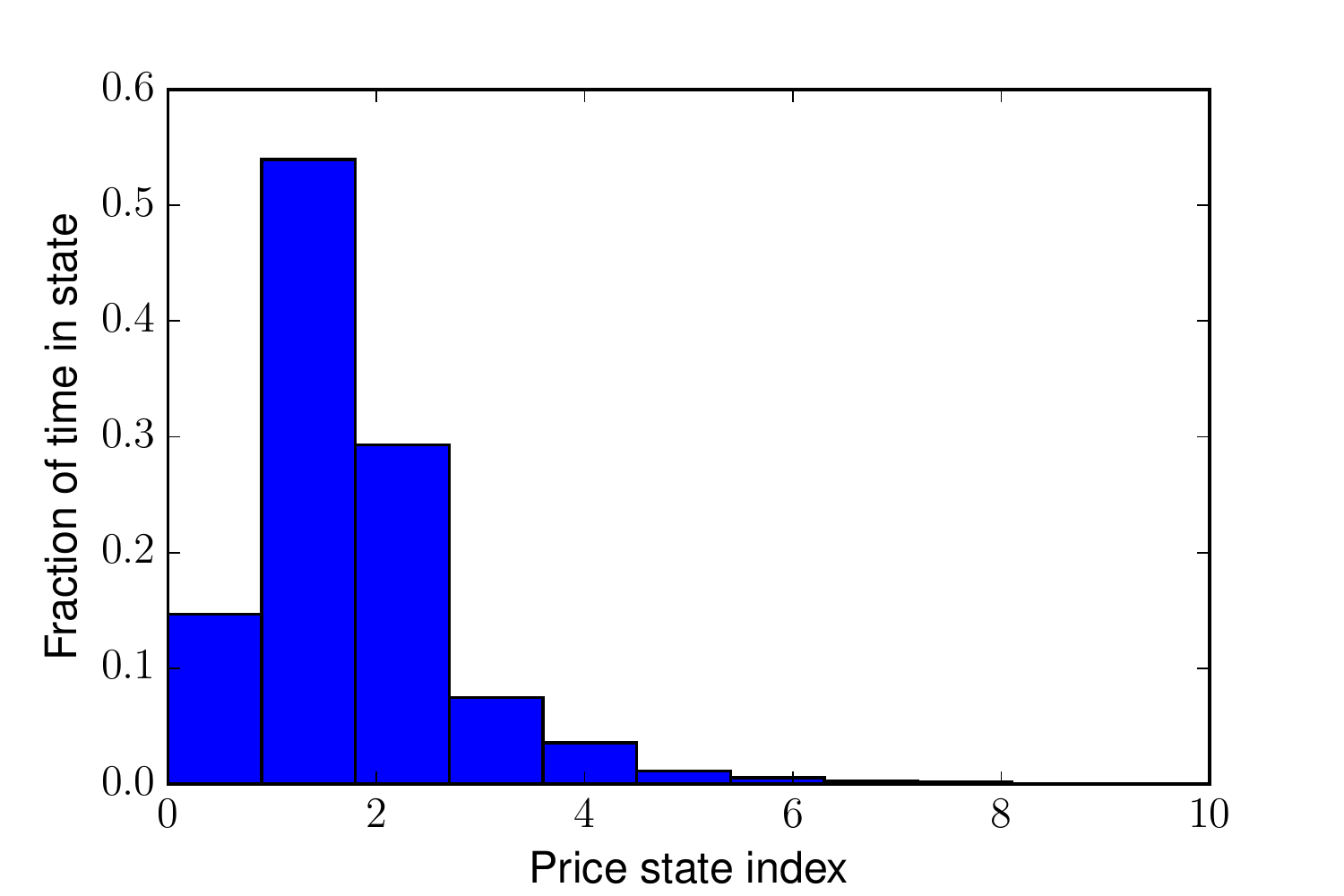}\,\,
	\includegraphics[width=0.45\linewidth]{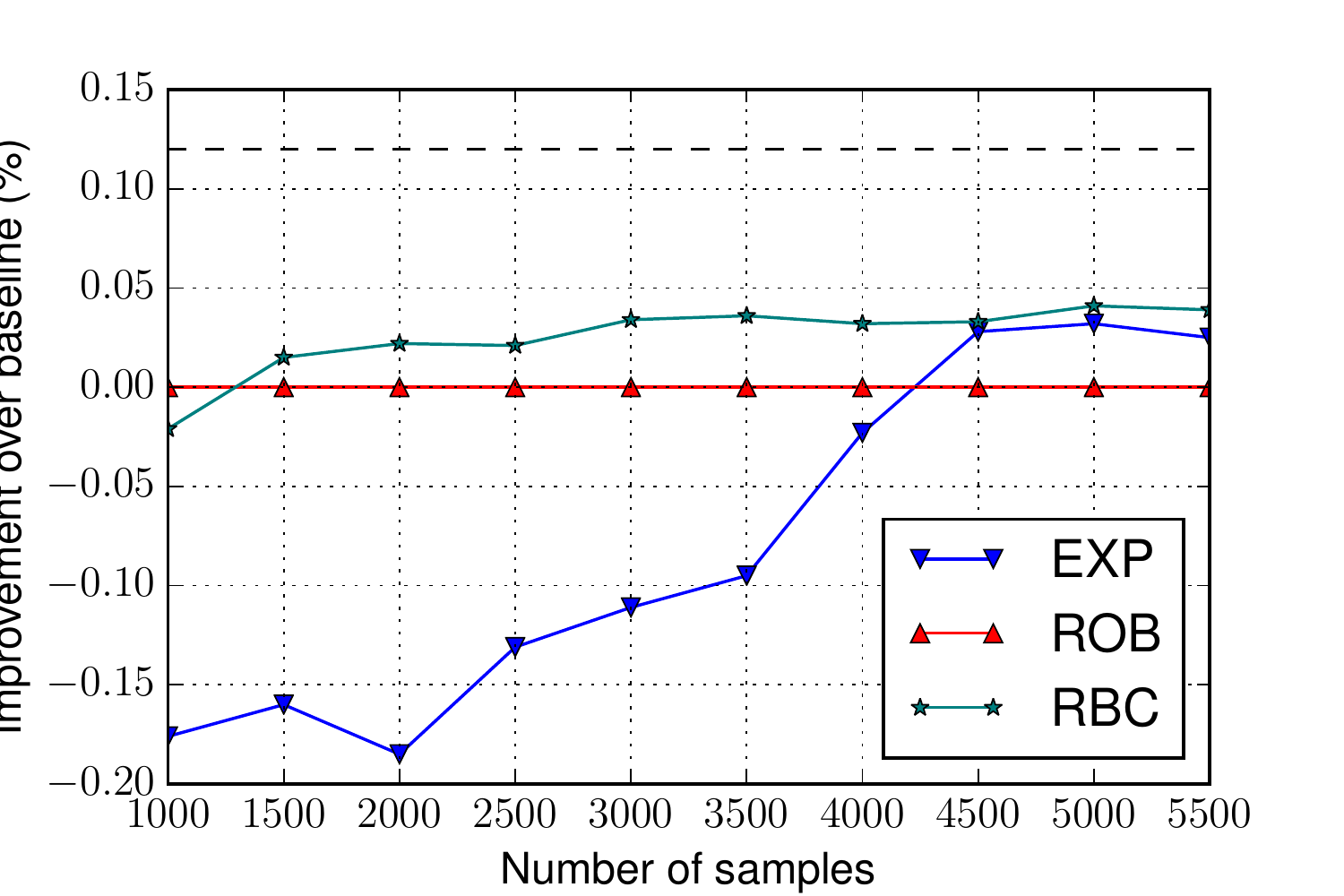}
	\vspace{-0.05in}
	\caption{\textbf{\textit{(left)}} Frequency of observed price indexes; each index corresponds to a discretized price level. \textbf{\textit{(right)}} Improvement over baseline policy as a function of the number of samples. }\label{fig:energy_results}
	\vspace{-0.15in}
\end{figure}

\section{Conclusion} 
\label{sec:conclusion}

In this paper, we studied the {\em model-based} approach to the fundamental problem of learning {\em safe} policies given a batch of data. A policy is considered {\em safe}, if it is guaranteed to have an improved performance over the baseline policy. Solving the problem of safety in sequential decision-making can immensely increase the applicability of the existing technology to real-world problems. We showed that the standard robust formulation may be overly conservative and formulated a better approach that interleaves an improved policy with the baseline policy, based on the error at each state. We proposed and analyzed an optimization problem based on this idea (see Eq.~\ref{eq:objective_robust_interleave}). We showed that the resultant problem may only have randomized solution policies, derived a performance bound for its solutions, and proved that solving it is NP-hard. Furthermore we proposed several approximate solutions and experimentally evaluated their performances. Since solving the optimization problem~\eqref{eq:objective_robust_interleave} is NP-hard, future work includes {\bf 1)} deriving approximate solution algorithms with tighter performance guarantees, and {\bf 2)} identifying specific structures in the the uncertainty set of transition probabilities that lead to a tractable solution algorithm.

\bibliographystyle{named}
\bibliography{library,LTV}

\newpage
\appendix
\section{Error Bound} 
\label{subsec:sampling_bounds}

Our goal here is to construct the error function $e$, when $\widehat{P}$ is estimated from the samples drawn from $P\opt$, such that we can guarantee that $P\opt \in \Nat(\widehat{P},e)$, with probability at least $1-\delta$. Let us assume that at each state-action pair $(x,a)\in\X\times\A$, we draw $N(x,a)$ samples from $P\opt(\cdot|x,a)$.

\begin{proposition}
If at each state-action pair $(x,a)\in\X\times\A$, we define $e(x,a)=\sqrt{\frac{2}{N(x,a)}\log\big(\frac{|\X||\A|2^{|\X|}}{\delta}\big)}$, then $P\opt \in \Nat(\widehat{P},e)$, with probability at least $1-\delta$.
\end{proposition}
\begin{proof}
From Theorem~2.1~in~\citet{weissman2003inequalities}, for each state-action pair $(x,a)\in\X\times\A$, we may write 
\begin{equation}
\label{eq:W1}
\mathbb{P}\left(||P\opt(\cdot\mid x,a) - \widehat{P}(\cdot\mid x,a)||_1\geq \epsilon\right) \leq (2^{|\X|}-2)\exp{-\frac{N(x,a)\epsilon^2}{2}}.
\end{equation}
Setting $\epsilon=\sqrt{\frac{2}{N(x,a)}\log\big(\frac{|\X||\A|2^{|\X|}}{\delta}\big)}$, we may rewrite~\eqref{eq:W1} as
\begin{align}
\label{eq:W2}
\mathbb{P}\Big(||P\opt(\cdot\mid x,a) - \widehat{P}(\cdot\mid x,a)||_1&\geq\sqrt{\frac{2}{N(x,a)}\log\big(\frac{|\X||\A|2^{|\X|}}{\delta}\big)}\Big) \nonumber \\
&\leq 2^{|\X|}\exp{-\frac{N(x,a)}{2}\times\frac{2}{N(x,a)}\log\big(\frac{|\X||\A|2^{|\X|}}{\delta}\big)} \nonumber \\ 
&= \frac{\delta}{|\X||\A|}\;.
\end{align}
From the definition of the uncertainty set $\Nat(\widehat{P},e)$ and by summing the error probability in~\eqref{eq:W2}, we obtain that $\mathbb{P}\big(P\opt\notin\Nat(\widehat{P},e)\big)\leq\delta$. 
\end{proof}

\newpage
\section{Proof of Lemma~\ref{cor:policy_tran_error}}
\label{sec:tech-lemma}

for which the following technical lemma (whose proof can be found in  Appendix~\ref{sec:tech-lemma}) is used in the analysis.

Before proving Lemma~\ref{cor:policy_tran_error}, we first prove the following lemma.

\begin{lemma} 
\label{lem:bound_transitions}
For any policy $\pi\in\Pi_R$, consider two transition probability matrices $P_1$ and $P_2$ and two reward functions $r_1$ and $r_2$ corresponding to $\pi$. Let $\val_1$ and $\val_2$ be the value functions of the policy $\pi$ given $(P_1,r_1)$ and $(P_2,r_2)$, respectively. Under the assumption that for any state $x\in\X,\;\| P_1(\cdot | x) - P_2(\cdot | x) \|_1 \le e(x)$, we have
\begin{equation*}
(\eye - \disc P_1)^{-1} \Bigl( r_1 - r_2 - \frac{\gamma \rmax}{1-\gamma} e \Bigr) \le \val_1 - \val_2 \le (\eye - \disc P_1)^{-1} \Bigl( r_1 - r_2 + \frac{\gamma \rmax}{1-\gamma} e  \Bigr),
\end{equation*}
where $e$ is the vector of $e(x)$'s. 
\end{lemma}

\begin{proof}
The difference between the two value functions can be written as 
\begin{align*}
\val_{1} - \val_{2} &= r_1 + \gamma P_1 \val_{1} - r_2 - \gamma P_2\val_{2} \\
&= r_1 + \gamma P_1 \val_{1} - r_2 - \gamma P_2 \val_{2} + \gamma P_1\val_{2} - \gamma P_1\val_{2} \\
&= (r_1 - r_2) + \gamma P_1(\val_{1} - \val_{2}) + \gamma (P_1 - P_2) \val_{2} \\
&= (\eye - \gamma P_1)^{-1} \left[r_1 - r_2 + \gamma (P_1 - P_2)\val_{2}\right].
\end{align*}
Now using the Holder's inequality, for any $x\in\X$, we have
\begin{equation*}
\left| \big(P_1(\cdot|x) - P_2(\cdot|x)\big)\tr \val_{2} \right| \le \left\| P_1(\cdot|x) - P_2(\cdot|x) \|_1 \right\| 
\val_{2} \|_\infty \le  e(x) \| \val_{2} \|_\infty \le  e(x) \frac{R_{\max}}{1-\gamma}.
\end{equation*}
The proof follows by uniformly bounding $(P_1 - P_2)\val_{2}$ from the above inequality and from the monotonicity of $(\eye - \gamma P_1)^{-1}$.
\end{proof}

\begin{lemma} 
	\label{cor:policy_tran_error}
	The difference between the returns of a policy $\pol$ in two MDPs parameterized by $P^\star,\nat\in\Nat$ is bounded as
	\begin{equation*}
	|\return(\pol,P^\star) - \return(\pol,\nat)| \le \frac{2\gamma\rmax}{1-\gamma} \, p_0\tr (\eye - \disc P^\star_\pol)^{-1} e_\pol,
	\end{equation*}
	where $P^\star_\pol$ and $e_\pol$ are the transition probability matrix and error function (between $P^\star$ and $\xi$, see Eq.~\ref{eq:ass1}) of policy $\pol$. 
\end{lemma} 
\begin{proof}
Lemma~\ref{cor:policy_tran_error} is the direct consequence of Lemma~\ref{lem:bound_transitions} with the fact that for any $(x,a)\in\X\times\A$ and any $\xi\in\Xi(\widehat{P},e)$, from Assumption~\ref{asm:error} and the construction of $\Nat(\widehat{P},e)$, we have 
\begin{align*}
\| P\opt(\cdot|x,a) - \xi(\cdot|x,a)\|_1 &= \| P\opt(\cdot|x,a) - \widehat{P}(\cdot|x,a) + \widehat{P}(\cdot|x,a) - \xi(\cdot|x,a) \|_1 \nonumber \\
&\le \| P\opt(\cdot|x,a) - \widehat{P}(\cdot|x,a) \|_1 + \| \widehat{P}(\cdot|x,a) - \xi(\cdot|x,a) \|_1  \nonumber \\
&\le 2 e(x,a)\;.
\end{align*}
\end{proof}


\newpage
\section{Proof of Theorem~\ref{thm:perf_robust_interleave}}
\label{sec:proof-bound}

To prove the safety of $\polrobi$, note that the objective in~\eqref{eq:objective_robust_interleave} is always non-negative, since the baseline policy $\polbase$ is a feasible solution. Thus, we obtain the safety condition by simple algebraic manipulation as follows:
\begin{equation}
\label{eq:performance_loss_robi}
\return(\polrobi,P^\star) - \return(\polbase,P^\star) \ge \min_{\nat\in\Nat} \Big( \return(\polrobi, \nat) - \return(\polbase, \nat) \Big) = \max_{\pi\in\Pi_R}\min_{\nat\in\Nat} \Big( \return(\pi, \nat) - \return(\polbase, \nat) \Big) \ge 0 ~. 
\end{equation}
Now we prove the performance bound. 
From \cref{cor:policy_tran_error}, for any policy $\pol$, we may write
\begin{equation} 
\label{eq:robi_proof_upper}
\max_\xi \Bigl|  \rho(\pol, \nat) - \rho(\pol,P^\star) \Bigr|  \le \frac{2\gamma \rmax}{1-\gamma} \, p_0\tr (\eye - P_\pol\opt )^{-1} \, e_\pol = \frac{2\gamma \rmax}{(1-\gamma)^2} \, \| e_\pol \|_{1,u^\star_\pi} ~,    
\end{equation}
where $u^\star_\pol$ is state occupancy distribution of policy $\pol$ in the true MDP $P^\star$, defined as
\begin{equation*}
u^\star_\pi = (1-\gamma) (\eye - \gamma P^{\star\top}_\pol)^{-1} p_0.
\end{equation*}
We are now ready to show a bound on the performance loss of $\polrobi$ through the following set of inequalities: 
\begin{align}
\label{eq:proof5-1}
\Phi(\pi_S) &= \rho(\pi^\star,P^\star) - \rho(\pi_S,P^\star) = \rho(\pi^\star,P^\star) - \rho(\pi_S,P^\star) + \rho(\pi_B,P^\star) - \rho(\pi_B,P^\star) \nonumber \\ 
&\le \rho(\pi^\star,P^\star) - \rho(\pi_B,P^\star) - \min_\xi\Big(\rho(\pi_S,\xi) - \rho(\pi_B,\xi)\Big) \nonumber \\ 
&\le \rho(\pi^\star,P^\star) - \rho(\pi_B,P^\star) - \min_\xi\Big(\rho(\pi^\star,\xi) - \rho(\pi_B,\xi)\Big) \nonumber \\
&\le \rho(\pi^\star,P^\star) - \rho(\pi_B,P^\star) - \min_\xi\rho(\pi^\star,\xi) + \max_\xi\rho(\pi_B,\xi) \nonumber \\
&= \max_\xi\Big(\rho(\pi^\star,P^\star) - \rho(\pi^\star,\xi)\Big) + \max_\xi\Big(\rho(\pi_B,\xi) - \rho(\pi_B,P^\star)\Big) \nonumber \\
&\stackrel{\text{(a)}}{\leq} \frac{2\gamma \rmax}{(1-\gamma)^2}\Big(\| e_{\pi^\star} \|_{1,u^\star_{\pi^\star}} + \| e_{\pi_B} \|_{1,u^\star_{\pi_B}}\Big)~,
\end{align}
where {\bf (a)} is by applying \eqref{eq:robi_proof_upper} to the two $\max$ terms on the RHS of the inequality. 

The final bound is obtained by combining \eqref{eq:proof5-1} and the fact that $\rho(\pi_S,P^\star)\geq\rho(\pi_B,P^\star)$, and as a result, $\Phi(\pi_S)\leq\Phi(\pi_B)$.

	
To prove the tightness of the bound, we use the example depicted in \cref{fig:example_tight_bound}. The initial state is $x_0$, actions are $a_1,a_2$, the transitions are deterministic, and the leafs represent absorbing states with the given return. We denote by $P^\star$, the transitions of the true MDP, and by $\nat_1$, the worst-case transitions in the uncertainty set $\Nat(\widehat{P},e)$. Finally the baseline policy $\pi_B$ takes action $a_1$ in state $x_0$ and shown by double edges in Figure~\ref{fig:example_tight_bound}. It is clear that the optimal policy $\pi^\star$ is the one that takes action $a_2$ in state $x_0$. The return of this policy is $\rho(\pi^\star,P^\star)=1+2\epsilon$. It is also straightforward to derive that the policy $\pi_S$ that takes action $a_1$ in state $x_0$ (as shown in Figure~\ref{fig:example_tight_bound}) is a solution to~\eqref{eq:objective_robust_interleave}. The return of this policy is $\rho(\pi_S,P^\star)=1$ and its performance loss is $\Phi(\pi_S)=\rho(\pi^\star,P^\star)-\rho(\pi_S,P^\star)=2\epsilon$. 

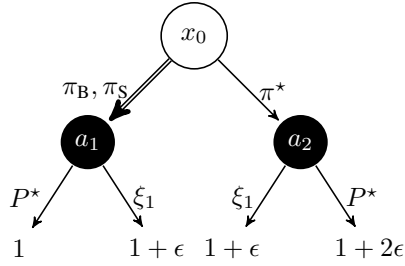
\begin{figure}[h]
\centering
\begin{tikzpicture}[->,>=stealth',shorten >=1pt,auto,node distance=2cm,semithick]
\tikzstyle{action}= [circle,fill=black,text=white]
	
\node[state]        (s1)            	                         {$x_0$};
\node[action]      (a1)         [below left of=s1]     	{$a_1$};
\node[action]      (a2)         [below right of=s1]    {$a_2$};
	
\node (o11) [below left of=a1,xshift=5mm] {$1$};
\node (o12) [below right of=a1,xshift=-5mm] {$1+\epsilon$};
\node (o21) [below left of=a2,xshift=5mm] {$1+\epsilon$};
\node (o22) [below right of=a2,xshift=-5mm] {$1+2\epsilon$};
	
\path (s1) edge [double] node [left]{$\pol_{\text{B}},\pol_{\text{S}}$} (a1) ;
\path (a1) edge node [left]{$P^\star$} (o11);
\path (a1) edge node [right]{$\nat_1$} (o12);
	
\path (s1) edge node [right]{$\pol\opt$} (a2) ;
\path (a2) edge node [left]{$\nat_1$} (o21);
\path (a2) edge node [right]{$P^\star$} (o22);
	
\end{tikzpicture}
\caption{Example showing the tightness of the bound in Theorem~\ref{thm:perf_robust_interleave}.} \label{fig:example_tight_bound}
\end{figure}

Now let us set $\epsilon$ in the leafs of Figure~\ref{fig:example_tight_bound} to $\epsilon=\frac{2\gamma \rmax}{(1-\gamma)^2} \, \| e_{\pi^\star} \|_{1,u^\star_{\pi^\star}}$. Note that this is the value given by~\eqref{eq:robi_proof_upper} for $\pi=\pi^\star$. This gives us the tightness proof assuming that $\widehat{P}$ is such that $\| e_{\pi^\star} \|_{1,u^\star_{\pi^\star}}$ and $\| e_{\pi_B} \|_{1,u^\star_{\pi_B}}$ have similar values, and $1+2\epsilon$ is a valid return value, i.e.,~$1+2\epsilon\leq\frac{\rmax}{1-\gamma}$. 


\newpage
\section{Proof of Theorem~\ref{thm:complexity_joint}}
\label{sec:proof-complexity}

\begin{figure}
    \centering
    \begin{tikzpicture}[->,>=stealth',shorten >=1pt,auto,node distance=1.5cm,semithick]
    \tikzstyle{action}= [circle,fill=black,text=white]
    
    \node[state,initial] (l11t)                                                          {$l_{11}^T$};
    \node[state]           (l11f)        [right= 0.6cm of l11t]           {$l_{11}^F$};        
    \node[state]             (l12t)        [below of=l11t]                    {$l_{12}^T$};
    \node[state]           (l12f)        [below of= l11f]                   {$l_{12}^F$};        
    \node[state]             (l13t)        [below of=l12t]                    {$l_{13}^T$};
    \node[state]           (l13f)        [below of= l12f]                   {$l_{13}^F$};        
    
    \node[state]            (l21t)         [right= 1.3cm of l11f]              {$l_{21}^T$};
    \node[state]           (l21f)        [right= 0.6cm of l21t]            {$l_{21}^F$};        
    \node[state]            (l22t)         [below of=l21t]                        {$l_{22}^T$};
    \node[state]           (l22f)        [below of= l21f]                       {$l_{22}^F$};        
    \node[state]           (l23t)         [below of=l22t]                        {$l_{23}^T$};
    \node[state]           (l23f)        [below of= l22f]                       {$l_{23}^F$};            
    
    \node[accepting,state] (end) [right=1.3cm of l21f] {$-1$};
    
    \begin{pgfonlayer}{background}
    \filldraw [line width=1mm,black!2]
    (l11t.north -| l11t.west) rectangle (l11f.south -| l11f.east)
    (l12t.north -| l12t.west) rectangle (l12f.south -| l12f.east)
    (l13t.north -| l13t.west) rectangle (l13f.south -| l13f.east)
    
    (l21t.north -| l21t.west) rectangle (l21f.south -| l21f.east)
    (l22t.north -| l22t.west) rectangle (l22f.south -| l22f.east)
    (l23t.north -| l23t.west) rectangle (l23f.south -| l23f.east);
    \end{pgfonlayer}
    
    \path (l11t) edge node{1} (l11f) 
    (l11f) edge node{0} (l21t) 
    (l11t) edge node{0} (l12t)
    (l11f) edge node[above]{1} (l12t);
    \path (l12t) edge node{1} (l12f) 
    (l12f) edge node{0} (l21t) 
    (l12t) edge node{0} (l13t)
    (l12f) edge node[above]{1} (l13t);
    \path (l13t) edge node{0} (l13f) 
    (l13f) edge[below] node{1} (l21t) 
    (l13t) edge[loop below] node[above]{1} (l13t)
    (l13f) edge[loop below] node[above]{0} (l13f);
    \path (l21t) edge node{0} (l21f) 
    (l21f) edge node{1} (end) 
    (l21t) edge node{1} (l22t)
    (l21f) edge node[above]{0} (l22t);
    \path (l22t) edge node{1} (l22f) 
    (l22f) edge node{0} (end) 
    (l22t) edge node{0} (l23t)
    (l22f) edge node[above]{1} (l23t);
    \path (l23t) edge node{1} (l23f) 
    (l23f) edge node[below]{0} (end) 
    (l23t) edge[loop below] node[above]{0} (l23t)
    (l23f) edge[loop below] node[above]{1} (l23f);
    \end{tikzpicture}
    \caption{MDP $\mathcal{M}_1$ in \cref{thm:complexity_joint} that represents the optimization of $\return(\pol, \xi)$ over $\xi$.} \label{fig:nphard_mdp1}
\end{figure}

\begin{figure}
    \centering
    \begin{tikzpicture}[->,>=stealth',shorten >=1pt,auto,node distance=1.5cm,semithick]
    \tikzstyle{action}= [circle,fill=black,text=white]
    
    \node[state,initial]      (avar)                        {$a$};
    \node[state]       (l11t)      [right of=avar]                {$l_{11}^T$};
    \node[state]       (l11f)      [below of=l11t]               {$l_{11}^F$};

        \path (l11t) edge[loop above] node[below]{0} (l11t)
                 (l11f) edge[loop below] node[above]{0} (l11f);

    \node[state]       (l21t)       [right of=l11t]             {$l_{21}^T$};
    \node[state]       (l21f)       [right of=l11f]             {$l_{21}^F$};

    \path (l21t) edge[loop above] node[below]{0} (l21t)
             (l21f) edge[loop below] node[above]{0} (l21f);
    
    \node[state]      (bvar)    [right of=l21t]                {$b$};
    
    \node[state]      (l12t)       [right of=bvar]            {$l_{12}^T$};
    \node[state]      (l12f)       [below of=l12t]        {$l_{12}^F$};

    \path (l12t) edge[loop above] node[below]{0} (l12t)
             (l12f) edge[loop below] node[above]{0} (l12f);

    \node[state]      (l23t)       [right of=l12t]            {$l_{23}^T$};
    \node[state]      (l23f)       [below of=l23t]        {$l_{23}^F$};

        \path (l23t) edge[loop above] node[below]{0} (l23t)
                 (l23f) edge[loop below] node[above]{0} (l23f);
    
    \node[accepting,state]     (end)  [right of=l23t]             {$+1$};

    \path (avar) edge node {T} (l11t) 
            (l11t) edge node {$1$} (l21t)
            (l21t) edge node {$1$} (bvar);
    \path (avar) edge node [below] {F} (l11f) 
            (l11f) edge node {$1$} (l21f)
            (l21f) edge node {$1$} (bvar);

    \path (bvar) edge node {T} (l12t) 
            (l12t) edge node {$1$} (l23t)
            (l23t) edge node {$1$} (end);
    \path (bvar) edge node [below] {F} (l12f) 
            (l12f) edge node {$1$} (l23f)
            (l23f) edge node {$1$} (end);
    \end{tikzpicture}
    \caption{MDP $\mathcal{M}_2$ in \cref{thm:complexity_joint} representing the optimization of $\return(\polbase, \xi)$ over $\xi$.} \label{fig:nphard_mdp2}
\end{figure}

Assume a given fixed policy $\pol$. We start by showing the NP hardness of solving~\eqref{eq:obj_interleave_inner}:
\begin{equation*}
\min_{\xi\in\Nat}\big(\return(\pol, \xi) - \return(\polbase, \xi)\big)
\end{equation*}
by a reduction from the boolean satisfiability (SAT) problem. To simplify the exposition, we also illustrate the reduction on the following simple example SAT problem in a conjunctive normal form (CNF):
\begin{equation} \label{eq:sat_example}
(a \vee  b \vee \neg\, c) \wedge (\neg\, a \vee d \vee b)  = (l_{11} \vee l_{12} \vee l_{13}) \wedge (l_{21} \vee l_{22} \vee l_{23})~,
\end{equation}
where $a$, $b$, $c$, and $d$ are the variables, and $l_{ij}$ represent the $j$-th literal in $i$-th disjunction. 

As noted above, $\return(\pol, \xi)$ represents the return of a robust MDP. Recall that computing $\min_\xi \return(\pol, \xi)$ for a fixed $\pol$ is equivalent to computing a policy in a regular MDP with actions representing realizations of the transition uncertainty. Therefore, optimizing for $\xi$ in \eqref{eq:obj_interleave_inner} translates to finding a single policy $\xi$ for two MDPs---defined by $\pol$ and $\polbase$---that maximizes the difference between their returns $\return(\pol, \xi) - \return(\polbase, \xi)$. 

We reduce the SAT problem to the optimization over $\xi$ in \eqref{eq:obj_interleave_inner}. As described above, the value $\return(\pol, \xi)$ for a fixed $\pol$ can be represented as a return of some MDP $\mathcal{M}_1$ for a policy given by $\xi$. Similarly, the value $\return(\polbase, \xi)$ for a fixed $\polbase$ can be represented as a return of another MDP $\mathcal{M}2$. We describe the general reduction in detail below. Figures~\ref{fig:nphard_mdp1} and~\ref{fig:nphard_mdp2} illustrate the MDPs $\mathcal{M}_1$ and $\mathcal{M}_2$ respectively for the example in \eqref{eq:sat_example}.

MDPs $\mathcal{M}_1,\mathcal{M}_2$ share the same state and action sets. The actions represent the realization of uncertainty $\xi$ and are denoted by the edge labels. They are discrete and stand for the extreme points of feasible $\ell_1$ uncertainty sets. For ease of notation, we assume $\gamma = 1$ and states with double circles are terminal with rewards inscribed therein. All non-terminal transition have zero rewards.

The identical state set of both $\mathcal{M}_1$ and $\mathcal{M}_2$ are constructed as follows. There is one state for each variable $\textup{v}\in\{a,b,c,d\}$, and two states $\{l_{ij}^T,l_{ij}^F\}$ for every literal $l_{ij}$. Informally, actions $\{T,F\}$ for a variable state capture the value of that variable. Actions $\{0,1\}$ for a literal state $l_{ij}^T$ or $l_{ij}^F$ represent the value of the variable referenced by the literal. This is regardless of whether the literal is positive or negative. For example, when the variable in $l_{ij}$ is true, the action in $l_{ij}^T$ is $1$ and when the variable in $l_{ij}$ is false, the action in $l_{ij}^F$ is $1$. Two states per each literal are necessary in order to model the negation operation.

The transitions in MDPs $\mathcal{M}_1$ and $\mathcal{M}_2$ are constructed to guarantee that their returns are $-1$ and $+1$, respectively (and as a result the objective in \eqref{eq:obj_interleave_inner} is $-2$), only if the assignment to the literals satisfies the SAT problem. Note that the transitions for the negated literals, such as $l_{21}$ are different from the positive literals, such as $l_{11}$. This construction easily generalizes to any SAT problem in the CNF. Consider the example in~\eqref{eq:sat_example} and let $b=T$ (other variables can take any values). It can then be seen readily that the objective in~\eqref{eq:obj_interleave_inner} would be $-2$.

Let $\rho\opt$ be the optimal value of~\eqref{eq:obj_interleave_inner}. Then, to show the correctness of our reduction, we argue that $\rho\opt = -2$, if and only if the SAT problem is satisfiable. To show the reverse implication, assume that the SAT is satisfied for some assignment to variables and construct a policy $\bar{\xi}$ as follows:
\begin{equation*} 
\bar\xi(\textup{v}) = \begin{cases} T &\text{if } \textup{v} = \true \\ F &\text{otherwise} \end{cases} ,\quad\quad
\bar\xi(l_{ij}^T) = \begin{cases} 1 &\text{if } \textup{v}_{ij} = \true \\ 0 &\text{otherwise} \end{cases} ,\quad\quad
\bar\xi(l_{ij}^F) = \begin{cases} 0 &\text{if } \textup{v}_{ij} = \true \\ 1 &\text{otherwise} \end{cases}~,
\end{equation*}
where $\textup{v}_{ij}$ represents the value of the variable referenced by the corresponding literal $l_{ij}$, e.g.,~$\textup{v}_{11} = \textup{v}_{21} = a$ in~\eqref{eq:sat_example}. It can be readily seen that $\return(\pol, \bar{\xi}) = 1$ and $\return(\polbase, \bar{\xi}) = -1$, and thus, the implication that $\rho^\star=-2$ holds.

To show the forward implication, assume that for an optimal deterministic realization $\bar\xi$, we have $\return(\polbase, \xi\opt) = 1$ and $\return(\polbase, \xi\opt) = -1$, and thus, $\return\opt = -2$. We assign values to variables $\textup{v}$ as follows:
\begin{equation*}
\textup{v} = \begin{cases}
\text{true} &\text{if } \bar\xi(\textup{v}) = T~, \\
\text{false} &\text{otherwise}~.
\end{cases}
\end{equation*}
We have that $\return(\polbase,\bar\xi) = 1$ only if for every disjunction $i$ either {\bf 1)} there exists a positive literal $l_{ij}$ such that $\bar\xi(l_{ij}^T) = 1$ and $\bar\xi(l_{ij}^F) = 0$, or {\bf 2)} there exists a negative literal $l_{ij}$ such that $\bar\xi(l_{ij}^T) = 0$ and $\bar\xi(l_{ij}^F) = 1$. Assume without loss of generality that this is always the first literal $l_{i1}$.  Now, consider any positive $l_{i1} = \textup{v}$ and observe that $\bar\xi(l_{i1}^T) = 1$ and $\bar\xi(l_{i1}^F) = 0$. Because $\return(\polbase,\bar\xi) = 1$ only if $\bar\xi(\textup{v}) = T$, the disjunction $i$ is satisfied. The case of a negative $l_{i1}$ is analogous, and thus, the forward implication also holds.

The restriction to deterministic policies $\bar\xi$ in the forward implication argument can be lifted by considering a discount factor; in such case the maximal return in $\mathcal{M}_2$ may be achieved only by a deterministic policy. Then, appropriately increasing the return in $\mathcal{M}_2$ finishes the argument.

The argument above shows that the inner minimization problem in \eqref{eq:objective_robust_interleave} is NP hard. Recall that \eqref{eq:objective_robust_interleave} is stated as follows:
\[ 
\max_{\pol \in\Polrand}  \min_{\xi\in\Nat}\big(\return(\pol, \xi) - \return(\polbase, \xi)\big) 
\]
To prove the theorem, it simply remains to show that the outer maximization over $\pol$ does not make the problem any easier. To show this, we will construct a single robust MDP $\mathcal{R}$ such that a policy with the maximal improvement induces $\mathcal{M}_1$ as the robust optimization subproblem. Baseline policy $\polbase$ in $\mathcal{R}$ similarly induces $\mathcal{M}_2$. Then, the difference between improved and baseline policies is no greater than some threshold if and only if the SAT is satisfiable.

Construct the robust MDP $\mathcal{R}$ with the same state set as $\mathcal{M}_1$ and $\mathcal{M}_2$. There are two actions $a_1$ and $a_2$ in each state. Upon taking action $a_1$, the transitions are chosen according to $\mathcal{M}_1$ and the reward is as in $\mathcal{M}_1$. Upon taking action $a_2$, the transition and reward are given the same as in $\mathcal{M}_2$ minus $3$. Rewards in terminal states are not modified.

The baseline policy takes action $a_2$, i.e. $\polbase(x) = a_2$. Return of the baseline policy is in $[3\, k, 3\,k + 1]$ where $k$ is the sum of the number of distinct variables and literals in the CNF. 

Let the improvement policy $\pol'$ be $\pol'(x) = a_1$. It can be readily seen that this policy achieves the maximal improvement. This is because $\return(\pol',\xi) \in [0,-1]$ while the return of any other policy will be at most $-3$ (the return for $a_2$ in any state is $-3$).

To finish the proof, observe that when the SAT is satisfiable then:
\[ \max_{\pol \in\Polrand} \min_{\xi\in\Nat}\big(\return(\pol, \xi) - \return(\polbase, \xi)\big)  = 
\min_{\xi\in\Nat}\big(\return(\pol', \xi) - \return(\polbase, \xi)\big) = 3\,k - 2 ~.\]
This is true using the above argument concerning the optimal value of the inner minimization problem. On the other hand, when the SAT is unsatisfiable then by the same argument:
\[ \max_{\pol \in\Polrand} \min_{\xi\in\Nat}\big(\return(\pol, \xi) - \return(\polbase, \xi)\big)  = 
 \min_{\xi\in\Nat}\big(\return(\pol', \xi) - \return(\polbase, \xi)\big) > 3\,k - 2 ~.\]
This shows that deciding whether the optimal value of \eqref{eq:objective_robust_interleave} is greater than $3\,k - 2$ is as hard as solving the corresponding SAT.

\newpage
\section{Proof of Theorem~\ref{thm:perf-simulator}}
\label{sec:proof-perf-simulator}

From Lemma~\ref{cor:policy_tran_error} with $\pi=\pi_{\text{sim}}$ and $\xi=\widehat{P}$ we have
\begin{equation*}
\return(\pi_{\text{sim}},\widehat{P}) - \frac{\gamma\rmax}{1-\gamma}p_0^\top(\eye-\gamma P^\star_{\pi_{\text{sim}}})^{-1}e_{\pi_{\text{sim}}} \leq \return(\pi_{\text{sim}},P^\star).
\end{equation*}
Thus, we may write
\begin{align*}
\Phi(\pi_{\text{sim}}) &\stackrel{\Delta}{=} \return(\pi^\star,P^\star) - \return(\pi_{\text{sim}},P^\star) \le \return(\pi^\star,P^\star) - \return(\pi_{\text{sim}},\widehat{P}) + \frac{\gamma \rmax}{1-\disc} \indist\tr (\eye - \disc P^\star_{\pi_{\text{sim}}})^{-1} e_{\pi_\text{sim}} \\
&\stackrel{\text{(a)}}{\le} \return(\pi^\star,P^\star) - \return(\pi^\star,\widehat{P}) + \frac{\gamma \rmax}{1-\disc} \indist\tr (\eye - \disc P^\star_{\pi_{\text{sim}}})^{-1} e_{\pi_{\text{sim}}} \\
&\stackrel{\text{(b)}}{\le} \frac{\gamma \rmax}{1-\disc} \indist\tr \left[(\eye - \disc P^\star_{\pi^\star})^{-1}e_{\pi^\star} + (\eye - \disc P^\star_{\pi_{\text{sim}}})^{-1}e_{\pi_{\text{sim}}}\right] \\
&\stackrel{\text{(c)}}{\le} \frac{2\gamma \rmax}{(1-\disc)^2}\| e \|_\infty~,
\end{align*}
where each step follows because:
\begin{itemize}
\item[(a)] Optimality of $\pi_{\text{sim}}$ in the MDP with transition probabilities $\widehat{P}$.
\item[(b)] Application of Lemma~\ref{cor:policy_tran_error} with policy $\pi=\pi^\star$ and $\xi=\widehat{P}$. 
\item[(c)] For any policy $\pi\in\Pi_R$, we have that $\| \indist\tr (\eye - \disc P^\star_\pi)^{-1} \|_1 = 1/(1-\gamma)$, and from the application of the Holder's inequality.
\end{itemize}
	

\newpage	
\section{Performance Bound on the Solution of the Robust Algorithm}
\label{sec:proof-perf-robust}

\begin{theorem}	
\label{thm:perf_robust}
Given Assumption~\ref{asm:error}, the nonempty solution $\pi_R$ of Algorithm~\ref{alg:robust} is safe, i.e.,~$\rho(\pi_R,P^\star) \ge \rho(\pi_B,P^\star)$. Moreover, its performance loss $\Phi(\pi_R)$ satisfies
\begin{equation*}
\Phi(\pi_R) \stackrel{\Delta}{=} \return(\pi^\star,P^\star) - \return(\pi_R,P^\star) 
\le \min\! \left\{\frac{2 \gamma \rmax}{(1-\disc)^2}  \left(\| e_{\pi^\star} \|_{1,u^\star_{\pi^\star}}\!+ \| e_{\pi_B} \|_{1,u^\star_{\pi_B}}\right),\Phi(\pi_B)\right\}, 
\end{equation*}
where $u^\star_{\pi^\star}$ is the state occupancy distribution of the optimal policy $\pi^\star$ in the true MDP $P^\star$, and $\Phi(\polbase)=\return(\pi^\star,P^\star) - \return(\polbase,P^\star)$ is the performance loss of the baseline policy.
\end{theorem}
\begin{proof}
To prove the safety of $\pi_R$ and bound its performance loss, we need to upper and lower bound the difference between the performance of any policy $\pi$ in the true MDP $P^\star$ and its worst-case performance in the uncertainty set $\Xi$, i.e.,~$\min_{\nat\in\Nat} \return\big(\pi,\nat\big)$. Since from Assumption~\ref{asm:error}, we have $P^\star\in\Nat$, we may write
\begin{equation} 
\label{eq:rob_proof_upper}
\min_{\nat\in\Nat}\return\big(\pi,\nat\big) \le \return(\pi,P^\star).
\end{equation} 
Now let $\bar{\xi}\in\Xi(\widehat{P},e)$ be the minimizer in $\min_{\nat\in\Nat}\return\big(\pi,\nat\big)$. The minimizer exists because of the continuity and compactness of the uncertainty set. 
Applying Lemma~\ref{cor:policy_tran_error} with $\nat=\bar{\xi}$, for any policy $\pi\in\Pi_R$, we obtain
\begin{equation} \label{eq:rob_proof_lower}
\return(\pi,P^\star) - \return\big(\pi,\bar{\xi}) = \return(\pi,P^\star) - \min_{\nat\in\Nat}\return\big(\pi,\nat\big) \le \frac{2\gamma \rmax}{1-\gamma} \, p_0\tr (\eye - \gamma P^\star_\pi)^{-1}e_\pi = \frac{2\gamma \rmax}{(1-\gamma)^2} \, \| e_\pi \|_{1,u^\star_\pi},    
\end{equation}
where $u^\star_\pi= (1-\gamma) p_0^\top(\eye - \gamma P^\star_\pi)^{-1}$ is the state occupancy distribution of policy $\pi$ in the true MDP $P^\star$. \\

{\bf To prove the safety of the returned policy $\pi_{R}$:} Consider the two cases on Line~\ref{ln:algrob_condition} of Algorithm~\ref{alg:robust}. When the condition is satisfied, i.e.,~$\rho_0>\max_{\nat\in\Nat}\return\big(\pi_B,\nat\big)$, we have 
\begin{equation*}
\return(\pi_B,P^\star) \leq \max_{\nat\in\Nat}\return\big(\pi_B,\nat\big) < \underbrace{\min_{\nat\in\Nat}\return\big(\pi_0,\nat\big)}_{\rho_0} \leq \return(\pi_0,P^\star), 
\end{equation*}
where the last inequality comes from~\eqref{eq:rob_proof_upper}, and thus, the policy $\pi_R=\pi_0$ is safe. When the condition is violated, then $\pi_R$ is simply $\pi_B$, which is safe by definition. \\

{\bf To derive a bound on the performance loss of the returned policy $\pi_R$:} Consider also the two cases on Line~\ref{ln:algrob_condition} of Algorithm~\ref{alg:robust}. When the condition is satisfied, using~\eqref{eq:rob_proof_upper}, we have
\begin{equation*}
\Phi(\pi_R) \stackrel{\Delta}{=} \return(\pi^\star,P^\star) - \return(\pi_R,P^\star) = \return(\pi^\star,P^\star) - \return(\pi_0,P^\star) \le \return(\pi^\star,P^\star) - \min_{\nat\in\Nat}\return\big(\pi_0,\nat\big)~,
\end{equation*}
and when the condition is violated, we have 
\begin{equation*}
\Phi(\pi_R) \stackrel{\Delta}{=} \return(\pi^\star,P^\star) - \return(\pi_R,P^\star) = \return(\pi^\star,P^\star) - \return(\pi_B,P^\star)~. 
\end{equation*}
Since when the condition is satisfied on Line~\ref{ln:algrob_condition} of Algorithm~\ref{alg:robust}, we have
\begin{equation*}
\min_{\nat\in\Nat}\return\big(\pi_0,\nat\big)>\max_{\nat\in\Nat}\return\big(\pi_B,\nat\big)
\end{equation*}
in both cases on Line~\ref{ln:algrob_condition} of Algorithm~\ref{alg:robust}, we may write
\begin{equation*} 
\Phi(\pi_R) \le \min\left\{\return(\pi^\star,P^\star) - \min_{\nat\in\Nat}\return\big(\pi_0,\nat\big)+\max_{\nat\in\Nat}\return\big(\pi_B,\nat\big)-\return(\pi_B,P^\star)~,~\overbrace{\return(\pi^\star,P^\star) - \return(\pi_B,P^\star)}^{\Phi(\pi_B)}\right\}~. 
\end{equation*}
The first term in the minimum can be written as
\begin{gather*}
\return(\pi^\star,P^\star) - \min_{\nat\in\Nat}\return\big(\pi_0,\nat\big)+\max_{\nat\in\Nat}\return\big(\pi_B,\nat\big)-\return(\pi_B,P^\star) \\
 \stackrel{\text{(a)}}{\le} 
\return(\pi^\star,P^\star) - \min_{\nat\in\Nat}\return\big(\pi^\star,\nat\big)+\max_{\nat\in\Nat}\return\big(\pi_B,\nat\big)-\return(\pi_B,P^\star) \\
\stackrel{\text{(b)}}{\le} \frac{2\gamma \rmax}{(1-\gamma)^2} \, \| e_{\pi^\star} \|_{1,u^\star_{\pi^\star}}+\frac{2\gamma \rmax}{(1-\gamma)^2} \, \| e_{\pi_B} \|_{1,u^\star_{\pi_B}},
\end{gather*}
where {\bf (a)} follows from $\pi_0$ being the solution to~\eqref{eq:objective_robust_interleave}, and thus, being the maximizer of $\min_{\nat\in\Nat}\return\big(\pi,\nat\big)$, and {\bf (b)} is from~\eqref{eq:rob_proof_lower} with $\pi=\pi^\star$ and $\pi=\pi_B$. 
\end{proof}


\newpage
\section{Solving the Reward-Adjusted MDP}
\label{sec:reward-adjusted}

In this section, we describe and analyze another simple method for computing safe policies that we did not include it in Section~\ref{sec:algorithms} due to space limitations, and show how it can be interpreted as an approximation of our proposed baseline regret minimization. This algorithm is based on solving a MDP with the same transition probabilities as the simulator, $\widehat{P}$, and rewards defined as $\widehat{r}(x,a)=r(x,a)-\frac{\gamma R_{\max}}{1-\gamma} e(x,a),\;\forall (x,a)\in\X\times\A$. We call this MDP, {\em reward-adjusted} (RaMDP), and denote its transition probabilities and rewards by $\widetilde{\xi}$. The unique property of RaMDP is that under Assumption~\ref{asm:error}, the performance of any policy $\pi$ in RaMDP is a lower-bound on its performance in the true MDP, i.e.,~$\rho(\pi,\widetilde{\nat})\le\rho(\pi,P^\star)$ (see Theorem~\ref{thm:perf_reward_adjusted}). Furthermore in comparison to the objective function of RMDP, the following proposition shows that $\rho(\pi,\widetilde{\nat})$ is always a lower-bound on $\min_{\nat\in\Nat} \return\big(\pi,\nat\big)$.

\begin{proposition}
	\label{prop:robust-vs-rewad}
	Given Assumption~\ref{asm:error}, for each policy $\pi$, we have $\;\min_{\nat\in\Nat} \return\big(\pi,\nat\big) \ge \return(\pi,\widetilde{\nat})$.
\end{proposition}
\begin{proof}
	Let $\bar{\xi}\in\Xi(\widehat{P},e)$ be the minimizer in $\min_{\nat\in\Nat}\return\big(\pi,\nat\big)$. The minimizer exists because of the continuity and compactness of the uncertainty set. From Lemma~\ref{lem:bound_transitions}, for each $\pi$, we may write
	\begin{equation*}
	\rho(\pi,\bar{\xi}) \ge \rho(\pi,\widehat{P}) - \frac{\gamma \rmax}{1-\gamma} p_0\tr (\eye - \gamma \widehat{P}_\pi)^{-1} e_\pi \stackrel{\text{(a)}}{=} \rho(\pi,\widetilde{\nat}),
	\end{equation*}
	where {\bf (a)} holds because $\widetilde{\nat}$ differs from $\widehat{P}$ only in its reward function, which is of the form $\widehat{r}_\pi=r_\pi-\frac{\gamma\rmax}{1-\gamma}e_\pi$.
\end{proof}

We conclude based on this proposition that the reward-adjusted method approximates the solution of the optimization problem~\eqref{eq:objective_robust_interleave} as
\begin{align}
\label{eq:RaMDP1}
\max_{\pi \in\Pi_R}\min_{\nat\in\Nat} \Bigl( \return(\pi, \nat) - \return(\pi_B, \nat)\Bigr) \nonumber &\geq \max_{\pi \in\Pi_R} \min_{\nat\in\Nat} \return(\pi, \nat)  -  \max_{\nat\in\Nat}\return(\pi_B, \nat) \nonumber \\ 
&\geq \max_{\pi \in\Pi_R}  \return(\pi, \widetilde{\nat}) -  \max_{\nat\in\Nat}\return(\pi_B, \nat),
\end{align}
and guarantees safety by designing $\pi$ such that the RHS of~\eqref{eq:RaMDP1} is always non-negative. Algorithm~\ref{alg:reward-adjusted} returns an optimal policy of the RaMDP $\widetilde{\nat}$, when the performance of this policy in $\widetilde{\nat}$ is better than the robust baseline performance $\max_{\nat\in\Nat}\rho(\pi_B,\nat)$, and returns $\pi_B$, otherwise. 

\IncMargin{1em}
\begin{algorithm}
	\SetKwInOut{Input}{input}\SetKwInOut{Output}{output}
	\Input{Simulated MDP $\widehat{P}$, baseline policy $\pi_B$, and the error function $e$} 
	\Output{Policy $\pi_{Ra}$}
	$\widehat{r}(x,a) \leftarrow r(x,a)-\frac{\gamma R_{\max}}{1-\gamma} e(x,a)$ \;
	$\pi_0 \leftarrow \arg \max_{\pi\in\Pi_R} \rho(\pi,\widetilde{\nat})$; $\quad\quad$ where $\;\widetilde{\xi}=(\widehat{r},\widehat{P})$ \\
	$\rho_0 \leftarrow \rho(\pi_0,\widetilde{\nat})$ \;
	\leIf{$\rho_0 > \max_{\nat\in\Nat}\rho(\pi_B,\nat)$}{
		$\pi_{Ra} \leftarrow \pi_0$}{
		$\pi_{Ra} \leftarrow\pi_B$ \label{ln:algexp_condition}}
	\Return $\pi_{Ra}$
	\caption{RaMDP-based Algorithm} 
	\label{alg:reward-adjusted}
\end{algorithm}

Since the performance of any policy in the RaMDP $\widetilde{\nat}$ is a lower-bound on its performance in the true MDP $P^\star$, it is guaranteed that the policy $\pi_{Ra}$ returned by Algorithm~\ref{alg:reward-adjusted} performs at least as well as the baseline policy $\pi_B$. Theorem~\ref{thm:perf_reward_adjusted} shows that $\pi_{Ra}$ is {\em safe} and quantifies its performance loss. 

\begin{theorem}
	\label{thm:perf_reward_adjusted}
	Given Assumption~\ref{asm:error}, the solution $\pi_{Ra}$ of Algorithm~\ref{alg:reward-adjusted} is safe, i.e.,~$\rho(\pi_{Ra},P^\star) \ge \rho(\pi_B,P^\star)$. Moreover, its performance loss $\Phi(\pi_{Ra})$ satisfies 
	\begin{equation*}
	\Phi(\pi_{Ra})\stackrel{\Delta}{=}\return(\pi^\star,P^\star) - \return(\pi_{Ra},P^\star) \le \min\!\left\{\frac{2 \gamma \rmax}{(1-\disc)^2}  \left(\| e_{\pi^\star} \|_{1,u^\star_{\pi^\star}}\!+ \| e_{\pi_B} \|_{1,u^\star_{\pi_B}}\right), \Phi(\pi_{B})\right\},
	\end{equation*}
	where $u^\star_{\pi^\star}$ is the state occupancy distribution of the optimal policy $\pi^\star$ in the true MDP $P^\star$, and $\Phi(\polbase)=\return(\pi^\star_{\nat\opt},P^\star) - \return(\polbase,P^\star)$ is the performance loss of the baseline policy.
\end{theorem}
\begin{proof}
	To prove the safety of $\pi_{Ra}$ and bound its performance loss, we need to upper and lower bound the difference between the performance of any policy $\pi$ in the true MDP $P^\star$ and its performance in $\widetilde{\nat}$, i.e.,~$\return(\pi,P^\star)-\return(\pi,\widetilde{\nat})$. These upper and lower bounds are obtained by applying Lemma~\ref{lem:bound_transitions} with $P_1=P^\star$, and $P_2=\widetilde{\nat}$ as follows:
	\begin{equation}
	\label{eq:AppC1}
	\return(\pi,P^\star) - \return(\pi,\widetilde{\nat}) \ge p_0\tr (\eye - \gamma P^\star_\pi)^{-1} \left(r_\pi - \widehat{r}_\pi - \frac{\gamma \rmax}{1-\gamma} e_\pi \right) \ge 0,
	\end{equation}
	where the second inequality in~\eqref{eq:AppC1} follows from the definition of the adjusted reward function $\widehat{r}$, and the fact that $(\eye - \gamma P^\star_\pi)^{-1}$ is monotone and $p_0$ is non-negative. Similarly, the upper-bound may be written as
	\begin{equation}
	\label{eq:exp_proof_lower}
	\return(\pi,P^\star) - \return(\pi,\widetilde{\nat}) \le \frac{2\gamma \rmax}{1-\gamma} p_0\tr (\eye - \gamma P^\star_\pi)e_\pi = \frac{2\gamma \rmax}{(1-\gamma)^2} \, \| e_\pi \|_{1,u^\star_\pi},    
	\end{equation}
	where $u^\star_\pi = (1-\gamma) p_0^\top(\eye - \gamma P^\star_\pi)^{-1}$ is the state occupancy distribution of policy $\pi$ in the true MDP $P^\star$. \\
	
	{\bf To prove the safety of the returned policy $\pi_{Ra}$:} Consider the two cases on Line~\ref{ln:algexp_condition} of Algorithm~\ref{alg:reward-adjusted}. When the condition is satisfied, we have 
	\begin{equation*}
	\return(\pi_B,P^\star)\leq \max_{\nat\in\Nat}\rho(\pi_B,\nat)< \return(\pi_0,\widetilde{\nat}) \le \return(\pi_0,P^\star),
	\end{equation*}
	where the last inequality comes from~\eqref{eq:AppC1}, and thus, the policy $\pi_{Ra}=\pi_0$ is {\em safe}. When the condition is violated, then $\pi_{Ra}$ is simply $\pi_B$, which is safe by definition. \\
	
	{\bf To derive a bound on the performance loss of the returned policy $\pi_{Ra}$:} Consider also the two cases on Line~\ref{ln:algexp_condition} of Algorithm~\ref{alg:reward-adjusted}. When the condition is satisfied, using~\eqref{eq:AppC1}, we have
	\begin{equation*}
	\Phi(\pi_{Ra}) \stackrel{\Delta}{=} \return(\pi^\star,P^\star) - \return(\pi_{Ra},P^\star) = \return(\pi^\star,P^\star) - \return(\pi_0,P^\star) \le \return(\pi^\star,P^\star) - \return(\pi_0,\widetilde{\nat}),
	\end{equation*}
	and when the condition is violated, we have 
	\begin{equation*}
	\Phi(\pi_{Ra}) \stackrel{\Delta}{=} \return(\pi^\star,P^\star) - \return(\pi_{Ra},P^\star) = \return(\pi^\star,P^\star) - \return(\pi_B,P^\star).
	\end{equation*}
	Since when the condition is satisfied on Line~\ref{ln:algexp_condition} of Algorithm~\ref{alg:reward-adjusted}, we have 
	\begin{equation*}
	\return(\pi_0,\widetilde{\nat})>\max_{\nat\in\Nat}\rho(\pi_B,\nat),
	\end{equation*}
	in both cases on Line~\ref{ln:algexp_condition} of Algorithm~\ref{alg:reward-adjusted}, we may write
	\begin{equation*}
	\Phi(\pi_{Ra}) \le \min\left\{\return(\pi^\star,P^\star) - \return(\pi_0,\widetilde{\nat})+\max_{\nat\in\Nat}\return\big(\pi_B,\nat\big)-\return(\pi_B,P^\star)\;,\;\overbrace{\return(\pi^\star,P^\star) - \return(\pi_B,P^\star)}^{\Phi(\pi_B)}\right\}.
	\end{equation*}
	The first term in the minimum may be written as
	\begin{equation*}
	\begin{split}
	&\return(\pi^\star,P^\star) - \return(\pi_0,\widetilde{\nat})+\max_{\nat\in\Nat}\return\big(\pi_B,\nat\big)-\return(\pi_B,P^\star)\\
	\stackrel{\text{(a)}}{\leq}& \return(\pi^\star,P^\star) - \return(\pi^\star,\widetilde{\nat}) +\max_{\nat\in\Nat}\return\big(\pi_B,\nat\big)-\return(\pi_B,P^\star)\stackrel{\text{(b)}}{\le} \frac{2\gamma \rmax}{(1-\gamma)^2} \, \| e_{\pi^\star} \|_{1,u^\star_{\pi^\star}}+\frac{2\gamma \rmax}{(1-\gamma)^2} \, \| e_{\pi_B} \|_{1,u^\star_{\pi_B}},
	\end{split}
	\end{equation*}
	where {(a)} follows from $\pi_0$ being an optimal policy of RaMDP $\widetilde{\nat}$ and {(b)} is from~\eqref{eq:exp_proof_lower} with $\pi=\pi^\star$ and $\pi=\pi_B$.
\end{proof}

Theorem~\ref{thm:perf_reward_adjusted} indicates that by this simple adjustment in the reward function of the simulator $\widehat{P}$, we may guarantee that our solution is {\em safe}. Moreover, it shows that the bound on the performance loss of $\pi_{Ra}$ is actually tighter than that for the solution $\pi_{\text{sim}}$ of the simulator, reported in Theorem~\ref{thm:perf-simulator}. 

While~\cref{alg:robust} is more complex than Algorithm~\ref{alg:reward-adjusted} (since solving a RMDP is more complicated than a standard MDP), Theorem~\ref{thm:perf_robust} does not show any advantage for $\pi_R$ over $\pi_{RA}$, neither in terms of safety nor in terms of the bound on its performance loss (compared to Theorem~\ref{thm:perf_reward_adjusted}). On the other hand, while Algorithm~\ref{alg:reward-adjusted} guarantees to yield a safe policy more efficiently than Algorithm~\ref{alg:robust}, from Proposition~\ref{prop:robust-vs-rewad} one notices that Algorithm~\ref{alg:reward-adjusted} may be overly conservative in many circumstances. This is because the adjustment of the reward function is based on the assumption that there exists a state with the maximum value of $\rmax / (1-\gamma)$ and that this state is accessible from all other states with reward $\rmax$. Thus, we may conclude that Algorithm~\ref{alg:robust} returns a less conservative safe policy (compared to Algorithm~\ref{alg:reward-adjusted}), with extra computational cost.

The experimental results of Section~\ref{sec:experiments} also show that the reward-adjusted solution of Algorithm~\ref{alg:reward-adjusted} can be extremely conservative.


\newpage
\section{Description of Experimental Domains} \label{app:domains}


\subsection{Grid Problem}  \label{app:grid_problem}

We now describe the grid problem in more detail. The state space in the problem comes from a two-dimensional grid: $\mathcal{S} = \{s_{ij} \ss i\in\mathcal{I}, j\in \mathcal{J}\}$; $i$ and $j$ represent a column and row respectively. Columns represent states of an interaction with the website, and rows represent more complex customer states, such as overall satisfaction. The dimensions are $|\mathcal{I}|= 12$ and $\mathcal{J}=3$.

There are 4 actions: $a_L$ for left, $a_R$ for right, $a_U$ for up, and $a_U$ for down. Rewards are independent of actions and depend only on states and only on the column: $x_{ij} = r_i$ where $r = [-1,1,2,3,2,1,-1,-2,-3,3,4,5]$. Actions left and right generally decrease and increase the column number; but can fail and in that case the transition is to a random column. The failure probability $z_j$ depends on the row $j$, with specific failure probabilities: $z=[0.9,0.2,0.3]$. If a transition fails, then the next state is chosen according to a fixed distribution which is generated a priori from a Dirichlet distribution. The distribution for first and last row are the same, and the middle row is the average of the two.

\cref{alg:transition} describes how the transition from a state is computed. The initial state is $s_{00}$.

\begin{algorithm}
	\KwData{Current state $s_{ij}$, action $a$, distributions $P_j$}
	\KwResult{Next state $s_{kl}$}
	
	\uIf{$\text{Random uniform from } [0,1] > z_j$}{
		\uIf{$a = a_R$}{
			$k \leftarrow i + 1$ \;
		}
		\uElseIf{$a = a_L$}{
			$k \leftarrow i - 1$ \;
		}
		\Else{
			$k \leftarrow \text{Random from } P_j$ \;
		}
		$k \leftarrow \max\{0, \min\{k, |\mathcal{I}|-1\} \}$ \;
	} 
	\Else{
		$k \leftarrow \text{Random from } P_j$ \;
	}
		
	\uIf{$a = a_U$}{
		$l \leftarrow j + 1$ \;
	}
	\uElseIf{$a = a_D$}{
		$l \leftarrow j - 1$ \;
	}
	\Else{
		$e \leftarrow \text{Random uniform from } [0,1]$ \;
		\uIf{$e \le 0.35$}{
			$l \leftarrow j + 1$ \;
		}
		\uElseIf{$e \le 0.7$}{
			$l \leftarrow j - 1$ \;
		}
		\Else{
			$l \leftarrow j$ \;
		}
	}
	$l \leftarrow \max\{0, \min\{l, |\mathcal{J}|-1\} \}$ \;
	\Return $s_{kl}$\;
	\caption{Transitions for state and action.} \label{alg:transition}
\end{algorithm}


\subsection{Energy Arbitrage} \label{app:energy_arbitrage}

The energy arbitrage model is based on \cite{Petrik2015} using a discount factor $0.9999$. We summarize it here for ease of reference. Recall that even though the state and action spaces in this problem are continuous, we discretize them as described in \cref{sec:experiments}.

The problem represents an energy arbitrage model with multiple finite \emph{known} price levels and a stochastic evolution given a limited storage capacity. In particular, the storage is assumed to be an electrical battery that degrades when energy is stored or retrieved.  Energy prices are governed by a Markov process with states $\Theta$. There are two energy prices in each time step: $p^i : \Theta \rightarrow \RealPlus$ is the purchase (or input) price and $p^o: \Theta \rightarrow \RealPlus$ is the sell (or output) price. Energy prices $\theta$ vary between 0 and 10 and their evolution is governed by a martingale with a normal distribution around the mean.

We use $s$ to denote the available battery capacity with $s_0$ denoting the initial capacity. The current state of charge is denotes as $z$ or $y$ and must satisfy that $0 \le z_t \le s_t$ at any time step $t$. The action is the amount of energy to charge or discharge, which is denoted by $a$. A positive $a$ indicates that energy is purchased to charge the battery; a negative $a$ indicates the sale of energy. 

The battery storage degrades with use. The degradation is a function of the battery capacity when charged or discharged. We use a general model of battery degradation with a specific focus on Li-ion batteries. The degradation function $d(z,a) \in \RealPlus$ represent the battery capacity loss after starting at the state of charge $z \ge 0$ and charging (discharging if negative) by $a$ with $-z \le a \le s_0$.
This function indicates the loss of capacity, such that:
\[ s_{t+1} = s_t - d(z_t,a_t) \]

The state set in the Markov decision problem is composed of $(z,s,\theta)$ where $z$ is the state of charge, $s$ is the battery capacity, and $\theta\in\Theta$ is the state of the price process. The available actions in a state $(z,s,\theta)$ are $a$ such that $-z \le a \le s-z$. The transition is from $(z_t,s_t,\theta_t)$ to $(z_{t+1},s_{t+1},\theta_{t+1})$ given action $a_t$ is:
\begin{align*}
z_{t+1} &= z_t + a_t \\
s_{t+1} &= s_t - d(z_t,a_t)
\end{align*}
The probability of this transition is given by $P[\theta_{t+1} \vert \theta_t]$. The reward for this transition is:
\[ r((z_t,s_t,\theta_t), a_t) = \begin{cases} 
- a_t \cdot p^i - c^d \cdot d(z_t,a_t) &\text{if } a_t \ge 0 \\ 
- a_t \cdot p^o - c^d \cdot d(z_t,a_t) &\text{if } a_t < 0 \end{cases}.\]
That is, the reward captures the monetary value of the transaction minus a penalty for degradation of the battery. Here, $c^d$ represents the cost of a unit of lost battery capacity.

The Bellman optimality equations for this problem are:
\begin{equation}\label{eq:bellman_optimality_simple} 
\begin{aligned}
q_T(z,s,\theta) &= 0 \\
v_t(z,s,\theta_t) &= \min \bigl\{p^i_{\theta_t} \pos{a} + p^o_{\theta_t} \negt{a}  +  \\
& \quad + c^d \, d(z,a) +  \\
&\quad +q_t(z + a,s- d(z,a),\theta_t) :  \\
&\qquad :  a \in [ -z, s - z ] \bigr\}   \\
q_t(z , s, \theta_t) &= \lambda \cdot \operatorname{E}[v_{t+1}( z,s,\theta_{t+1}) ]
\end{aligned}
\end{equation}
where $\pos{a} = \max\{a, 0\}$ and $\negt{a} = \min\{a, 0\}$ and the expectation is taken over $P(\theta_{t+1} | \theta_t)$. 

Please see \cite{Petrik2015} for more details, including the price transition matrix.

\end{document}